\documentclass[journal]{IEEEtran}
\usepackage{graphicx, bm, amsmath, amsthm, mathtools, subfig, scrextend, array, setspace, comment, algorithm, textcomp, soul, color}
\usepackage{epstopdf}

\newcolumntype{M}[1]{>{\centering\arraybackslash}m{#1}}
\newtheorem{theorem}{Theorem}
\newtheorem{definition}{Definition}
\ifCLASSINFOpdf
\else
\fi
\begin{document}
%
\title{Robust Multi-Modal Sensor Fusion: An Adversarial Approach}
%
%
%

\author{Siddharth Roheda$^{\star}$,
        ~Hamid~Krim$^{\star}$,
        ~and~Benjamin~S.~Riggan$^{\dagger}$.\\
        $ ^{\star}$North Carolina State University at Raleigh, NC, USA\\
        $ ^{\dagger} $University of Nebraska-Lincoln, Lincoln, NE, USA
        
}

\maketitle

\begin{abstract}

In recent years, multi-modal fusion has attracted a

	lot of research interest, both in academia, and in industry. Multimodal fusion entails the combination of information from a set of different types of sensors. Exploiting complementary information from different sensors, we show that  target detection and classification problems can greatly benefit from this fusion approach and result in a performance increase. To achieve this gain, the information fusion from various sensors is shown to require some principled
strategy to ensure that additional information is constructively used,
and has a positive impact on performance. We subsequently demonstrate the viability of  the proposed fusion approach by weakening the strong dependence on  the functionality of all sensors, hence introducing additional flexibility in our solution and lifting the severe limitation in unconstrained surveillance
settings with potential environmental impact.
Our proposed data driven approach to multimodal
fusion, exploits selected optimal features from an estimated  latent space of data across all modalities. This
hidden space is learned via a generative network conditioned on
individual sensor modalities. The hidden space, as an intrinsic
structure, is then exploited in detecting damaged sensors,
and in subsequently safeguarding the performance of the fused
sensor system. Experimental results show that such an approach
can achieve automatic system robustness against noisy/damaged sensors.

\end{abstract}

%

%
\IEEEpeerreviewmaketitle

\section{Introduction}
%
%
%
%

\IEEEPARstart{S}{ensor} fusion is known to broadly classify fusion techniques into three classes. Data level fusion is used when combining diverse raw data, to subsequently proceed with inference. Feature level fusion first extracts information from raw data, and then merges these features to make coherent decisions. Fusion at the decision level allows each sensor to reach its own individual decision (on the target identity), prior to an optimal combination of these decisions. Decision level fusion has received a lot of attention when exploiting heterogeneous modalities, as it allows each modality to have an independent feature representation, thus providing additional information. A review of classical fusion approaches can be found in \cite{survey1, survey2}, and a more recent alternative data-driven perspective is provided by \cite{wang2018fusing}.

The viability of many fusion approaches strongly  hinges on the functionality of all the sensors, making it a fairly restrictive solution. The severity of this limitation is even more pronounced in unconstrained surveillance settings, where the environmental conditions have a direct impact on the sensors, and close manual monitoring is difficult or even impractical. Partial sensor failure can hence cause a major drop in performance in a fusion system if timely failure detection is not performed. Furthermore, even if these sensors are successfully detected, the common adopted solution is to ignore the damaged sensor, with a potential negative impact on the overall resulting performance (i.e. relative to when all sensors are functional). We consider exploiting the prior information about the relationship between these sensor modalities, typically available from past observations  during training of the fusion system, so that the latter can safeguard a high performance. There has recently been an interest in such a transfer of knowledge. In \cite{JHoffHalluc}, the authors introduce hallucination networks to address a missing modality at test time by distilling knowledge into the network during training. Here, a teacher-student network is implemented using an $L^2$ loss term (hallucination loss) to train the student network. This idea has been further extended in \cite{GHintonDistill, VapnikPrivInfo, roheda2018cross}. In \cite{roheda2018cross} where a Conditional Generative Adversarial Network (CGAN) was used to generate representative features for the missing modality. Furthermore, \cite{SRastMdl:CW} shows that considering interactions between modalities can often lead to a better feature representation. Recent work in Domain Adaptation \cite{gopalan2011domain, zheng2012grassmann, ni2013subspace} addresses differences between source and target domains. In these works, the authors attempt to learn intermediate domains (represented as points on the Grassman manifold in \cite{gopalan2011domain, zheng2012grassmann} and by dictionaries in \cite{ni2013subspace}) between the source domain and the target domain.

\textbf{Our Contributions:} In this paper, we propose a robust sensor fusion algorithm, that can detect damaged sensors on the fly, and take the required steps to safeguard detection performance. We use a special case of the Event Driven Fusion technique recently proposed in \cite{rohedaEDF, rohedaEDF-jour}, and modify it in order to include reliability of individual sensors. This reliability measure is adaptive, and accounts for the sensor condition during implementation. We also propose a data driven approach for learning the features of interest which were handcrafted in \cite{rohedaEDF}. In addition, we learn a hidden latent space between the sensor modalities, and the optimal features for classification are driven by the existence of this hidden space. Furthermore, it also provides robustness against damaged sensors. This hidden space is learned via a generative network conditioned on individual sensor modalities. In contrast to \cite{roheda2018cross}, we do not require any target feature space in order to learn the optimal hidden space estimate. The hidden space is structured so that it can accommodate both \textit{shared and private} features of sensor modalities. We forego the use of a-priori knowledge about sensor damage adopted in \cite{rohedaEDF-jour}, and can hence detect damaged sensors based on deviations in the generated hidden space. 

\section{Related Work}

\subsection{Model Based Sensor Fusion}
\label{Model_Based_Fusion}
A model-based sensor fusion approach was proposed in  \cite{thesis, li2001convex, florea2007critiques}.In this work, two different sensor fusion models were explored, namely, Similar Sensor Fusion (SSF), and Dissimilar Sensor Fusion (DSF). The SSF model assumes that the independent sensors in the network are similar to each other (eg. 5 radars looking at the same object). Additional sensors in this case do not provide any new information, but can be used to confirm information from other sensors. This model attempts to find a fusion result which is most consistent with all the individual sensor reports. On the other hand, the DSF model assumes that  all sensors observe dissimilar characteristics of the target, and hence each sensor provides novel information about the target. An additional sensor in this case generates increased clarity on the target identity. These two models should be viewed as ``extreme cases" of decision level fusion. There are many practical cases in which the sensors are neither completely similar nor completely dissimilar. Event Driven Fusion \cite{rohedaEDF, rohedaEDF-jour} addresses this limitation and is discussed below.
\subsection{Event Driven Fusion}
\label{EDF}
Event Driven Fusion \cite{rohedaEDF, rohedaEDF-jour} looks at fusion under a different light, by combining occurrences of events to reach a probability measure for a target identity. Each sensor is said to make a decision on the occurrence of certain events that it observes, rather than making a decision on the target identity. This technique also explores the extent of dependence between features being observed by the sensors, and hence generates more informed probability distributions over the events.

Consider the set of objects/targets, $O = \{o_1, o_2, ..., o_I\}$. Let the $k^{th}$ feature from the $l^{th}$ sensor be $F_k^l$. Then, a mutually exclusive set of events, $\Omega_k^l = \{a_{kj}^l\}_{j=1,...,J_{kl}}$, is defined over the feature $F_k^l$. Here, $a_{kj}^l$ is the $j^{th}$ event on $F_k^l$ and is described as $a_{kj}^l: F_k^l \in [u_j, v_j)$, $u_j \in {\rm I\!R^+}$, $v_j \in {\rm I\!R^+}$, and $v_j > u_j$. A probability report is generated by the $l^{th}$ sensor for each of its features, $R_k^l = \{\Omega_k^l, \sigma_B(\Omega_k^l), P_k^l\}$, where, $\sigma_B(\Omega_k^l)$ is the Borel sigma algebra of $\Omega_k^l$, and $P_k^l$ is the set of probabilities over all the events in $\sigma_B(\Omega_k^l)$.  An object/target is then defined as $o_i \in \sigma_B(\Omega)$, where, $\Omega = \Omega_1^1 \times ... \times \Omega_k^l \times ... \times \Omega_{K_L}^L$. The joint probability in the product space is determined as a convex combination of a distribution with minimal mutual information and one with maximal mutual information. For events $\gamma_k^l \in \sigma_B(\Omega_k^l)$
\begin{equation}
	\begin{split}
	P_\Omega(\gamma_1^1,...,\gamma_k^l,...,\gamma_{K_L}^L) =  \rho.P_{\Omega_{\text{MAXMI}}}(\gamma_1^1,...,\gamma_k^l,...,\gamma_{K_L}^L) +\\ (1-\rho).P_{\Omega_{\text{MINMI}}}(\gamma_1^1,...,\gamma_k^l,...,\gamma_{K_L}^L), 
	\end{split}
\end{equation}  
where $\rho \in [0,1]$ is a pseudo-measure of correlation between the features. $\rho \approx 1$ when the features are highly correlated, and $\rho = 0$ when they are independent.
For any object defined as a combination of events in the product space $\Omega$, $o_i \in \sigma_B(\Omega)$, rules of probability can then be used to determine the object probability. For instance, in a 2-D scenario, an object may be defined as a combination of events $\gamma_1 \in \sigma_B(\Omega_1)$, and $\gamma_2 \in \sigma_B(\Omega_2)$. The combination defined in the product space, $\Omega = \Omega_1 \times \Omega_2$, may be of the form $o: \{\gamma_1 \wedge \gamma_2\}$ or $o: \{\gamma_1 \vee \gamma_2\}$. Given the joint probability, $P_\Omega$, rules of probability can be used to determine the fused object probability as follows:
\begin{itemize}
	\item $o: \{\gamma_1 \wedge \gamma_2\}: P^f(o) = P_\Omega(\gamma_1, \gamma_2)$
	\item $o: \{\gamma_1 \vee \gamma_2\}: P^f(o) = P_1(\gamma_1) + P_2(\gamma_2) - P_\Omega(\gamma_1, \gamma_2)$
\end{itemize}
Where, $P_1(\gamma_1)$ and $P_2(\gamma_2)$ are the marginal probabilities for detection of the events $\gamma_1$ and $\gamma_2$ as seen by sensors $1$ and $2$.

\subsection{Generative Adversarial Networks}
The Adversarial Network was first introduced by Goodfellow et al. \cite{GoodfellowGAN} in 2014. In this framework, a generative model is pitted against an adversary: the discriminator. The generator aims to deceive the discriminator by synthesizing realistic samples from some underlying distribution. The discriminator on the other hand attempts to discriminate between a real data sample and that from the generator. Both these models are approximated by neural networks. When trained alternatively, the generator learns to produce random samples from the data distribution which are very close to the real data samples. Following this, Conditional Generative Adversarial Networks (CGANs) were proposed in \cite{MirzaCGAN}. These networks were trained to generate realistic samples from a class conditional distribution, by replacing the random noise input to the generator by some useful information (see Figure \ref{CGAN}). Hence, the generator now aims to generate realistic data samples, given the conditional information. CGANs have been used to generate random faces, given facial attributes \cite{CNNFaceGen} and to also produce relevant images given text descriptions \cite{CGANimgtotxt}.  
\begin{figure}
	\centering
	\includegraphics[width=0.4\textwidth]{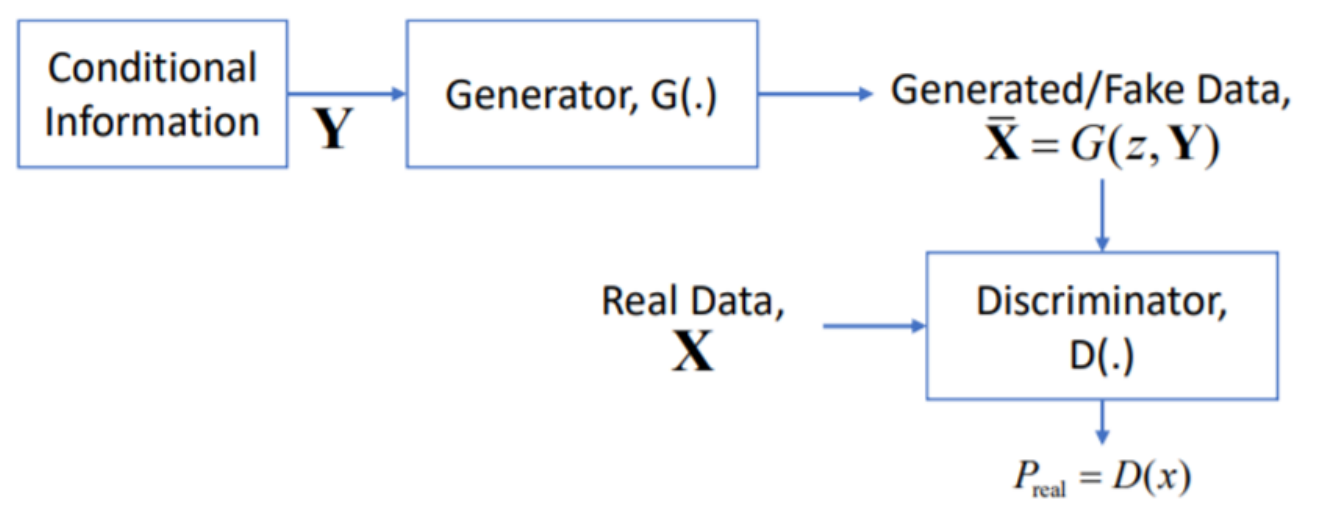}
	\caption{Conditional Generative Adversarial Networks}
	\label{CGAN}
\end{figure}

\section{Problem Formulation} \label{problemform}
\begin{figure*}[tbp]
	\centering
	\includegraphics[width=0.6\textwidth]{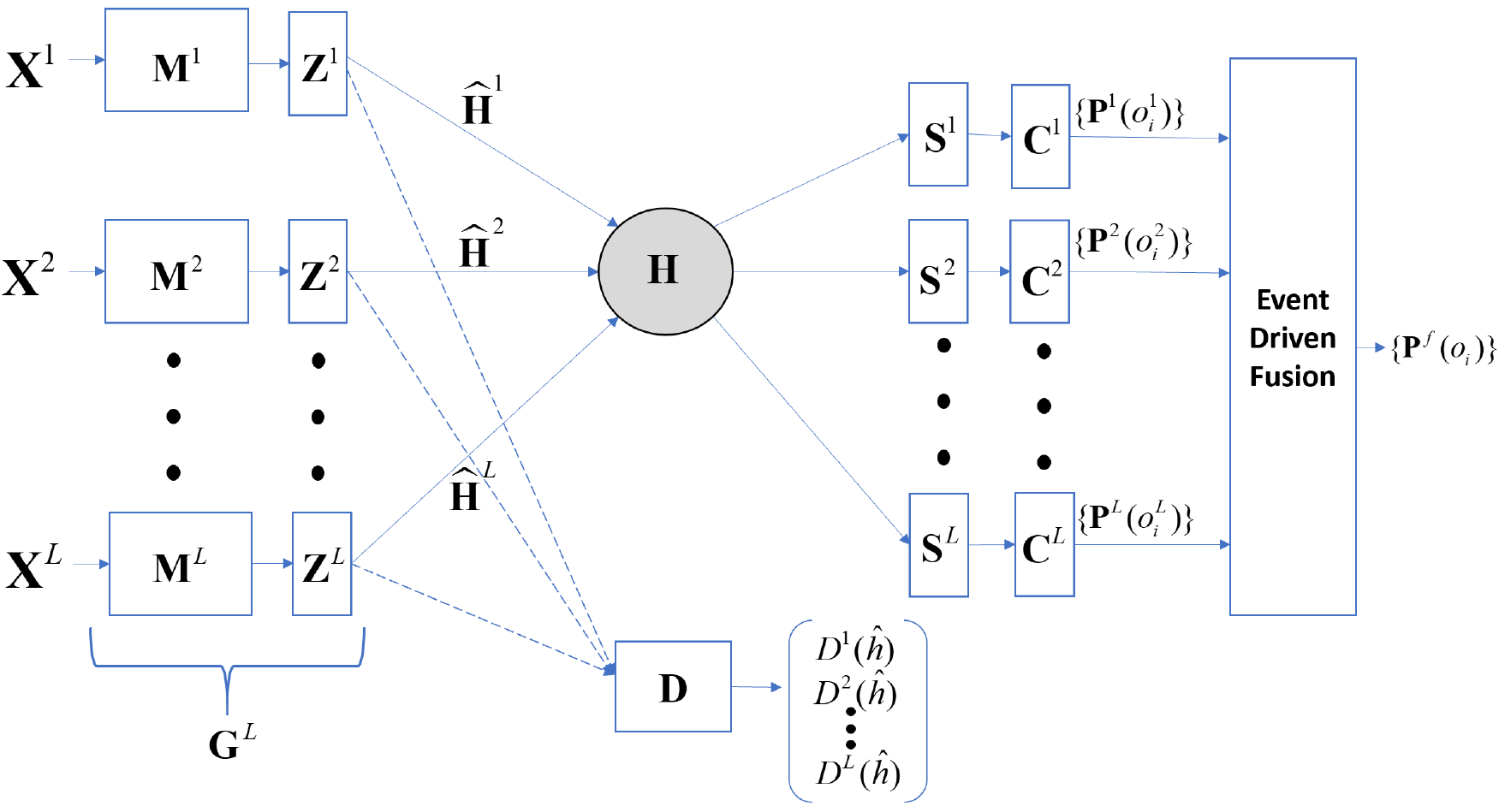}
	\caption{Block diagram of the proposed approach for sensor fusion}
	\label{FUSIONGANsBD}
\end{figure*}
Consider $L$ sensors surveilling an area of interest. We wish to detect/recognize targets, $O=\{o_1,o_2,...,o_I\}$ given the data collected by the sensors. Let the $n^{th}$ observation from the $l^{th}$ sensor be denoted by, $\bm{x^l_n} = \{\bm{x^l}_{\bm{n}_q}\}_{q = 1...d_l}$.
Let $\bm{X^l}_{d_l \times N} = \{ \bm{x_n^l} \}_{n = 1...N}$ be the set of $N$ observations, each of dimension $d_l$, made by the $l^{th}$ sensor.
	
We first seek to discover a hidden common space\footnote{In \cite{wang2018fusing}, such a space was referred to as an `information subspace'} $\bm{H}_{d_H \times N}$, of features captured by and hence present in the $L$ sensors, where $d_H$ is the dimension of the common subspace. This structure, unknown a priori, may include  both shared and non-shared features across the sensors. We refer to the non-shared features of specific sensors, as private. Given a hidden space, this thus amounts to being able to select from each sensor the optimal set of features for classification via a selection matrix, $\bm{S}^{\bm{l}}$,
\begin{gather}
	\label{FlHl}
	\begin{align}
	& \nonumber \forall l \in \{1,...,L\}, \\
	&\bm{F}^{\bm{l}}_{d \times N} = \bm{S^l}_{d \times d_H} \bm{H}_{d_H \times N}.
	\end{align}
\end{gather}
Since the hidden space represents the information shared by the sensors, there must also exist a mapping, $\bm{G^l}: \bm{X^l} \to \bm{H}$ such that,
\begin{gather}
\label{GANTx}
	\forall l \in \{1,...,L\}, \text{ } {\bm{\hat{H}^l}}_{d_H \times N} = \bm{G^l}(\bm{X^l}_{d_l \times N}) \approx \bm{H}_{d_H \times N}
\end{gather}
From Equations \ref{FlHl} and \ref{GANTx} we have,

\begin{gather}
\begin{align}
&\nonumber \forall l \in \{1,...,L\}, \\
&\bm{F^l}_{d \times N} = \bm{S^l}_{d \times d_H}[(\bm{G^l}(\bm{X^l}))_{d_H \times N}].
\end{align}
\end{gather}

The existence of such a hidden space makes it possible to detect damaged sensors, and safeguard the system performance against them. When the $l^{th}$ sensor is damaged, representative features for that sensor are reconstructed from the hidden space via the selection operator, $\bm{F^l = S^lH}$. Following the feature extraction, a linear classifier, $c_i^l(\bm{F^l}) = w_i^{l^T}\bm{F^l} + b_i^l$, is used to determine the classification score for the $i^{th}$ object as seen by the $l^{th}$ sensor. The probability of occurrence of this object is then determined as,
\begin{equation}
	P^l(o_i^l)=\frac{exp({w_i^{l^T}\bm{F^l}} + b_i^l)}{\sum_{m=1}^{I} exp({w_m^{l^T}\bm{F^l}} + b_m^l) }.
\end{equation}
Finally, given these probability reports, $R^l=\{P^l(o_i^l)\}$, the objective is to determine the fused probability report $R^f = \{P^f(o_i)\}$, which is achieved using a special case of Event Driven Fusion \cite{rohedaEDF, rohedaEDF-jour}. 

\section{Proposed Approach}
\label{Proposed_Method}
As discussed in the previous section, the hidden space, $\bm{H}$, can be estimated from the $l^{th}$ sensor observations as, $\bm{\hat{H}^l} = \bm{G^{l}}(\bm{X^l})$. 
The mapping $\bm{G^l}$ is approximated by a neural network and is realized as the generator of a Conditional Generative Adversarial Network (CGAN), that generates the estimate of the hidden space, $\bm{\hat{H}^l}$, while conditioned on the observations of the $l^{th}$ sensor. Hence, we will have $L$ generators that generate $L$ estimates of the hidden space.

The desired output of these generators is to generate the estimate, $\bm{\hat{H}^l}$, such that, $\forall l \in \{1,...,L\},$ $\bm{\hat{H}^l} = \bm{G^l(X^l)} \approx \bm{H}$. On the other hand, the discriminator attempts to correctly identify the source modality of the estimated hidden space. This results in a score assignment $D^l(\hat{h})$ for an estimate $\hat{h}$ generated by $G^l$. When updating the parameters for the $l^{th}$ generator, the hidden space estimates generated by all the other generators are assumed to be the target space, and the $l^{th}$ generator attempts to replicate these spaces, while at the same time attempting to generate an estimate that confuses the discriminator. 

The standard formulation for the Generative Adversarial Network \cite{GoodfellowGAN, MirzaCGAN} is known to have some instability issues \cite{WGAN}. Specifically, if the supports of the estimated hidden spaces are disjoint, which is highly likely when the inputs to the generators are coming from different modalities, a perfect discriminator is easily learned, and gradients for updating the generator may vanish. This issue was addressed in \cite{WGAN}, and solved by using a Wasserstein GAN. So, using the sensor observations as the conditional information in the WGAN formulation \cite{WGAN}, we have,
\begin{gather}
	\label{WGAN_eqn}
	\begin{align}
	&\nonumber \min_{\bm{G^l}}\max_{\bm{D}}  \sum_{l=1}^{L} V(\bm{G^l},\bm{D}),\\ 
	& \nonumber V(\bm{G^l},\bm{D}) = \sum_{\substack{m = 1 \\ m \neq l}}^{L} \{{\rm I\!E}_{\bm{G^{m}}(x^m) \texttildelow {\rm I\!P}_{\bm{H}}} [D^m(\bm{G^{m}}(x^m))] \\
	&\quad \quad \quad \quad \quad \quad \quad - {\rm I\!E}_{x^l \texttildelow {\rm I\!P}_{\bm{X^l}}} [D^m(\bm{G^{l}}(x^l))]\}.
	\end{align}
\end{gather}
The discriminator $\bm{D}$, in the above formulation is required to be compact and $K$-Lipschitz. This is done by clamping the weights of the discriminator to a fixed box (eg. $\theta_d \in [-0.01,0.01]$) \cite{WGAN}. 
The discriminator is updated once after every generator update. That is, for $L$ sensors, the discriminator is updated $L$ times for one update to the $l^{th}$ generator  (see Algorithm 1). While this causes the estimates to be close to each other, it does not guarantee that the generated hidden spaces share the same basis.

A similar idea toward finding a common representation between multiple modalities was very recently and independently proposed in \cite{CMGANs}. In this paper the authors attempt to find a common latent space between text and image modalities in order to perform image/text retrieval. Given a query from the image modality, the image is transformed into the latent common space, and the text with closest representation is retrieved.

In our work, we additionally ensure that the hidden space estimates share a common basis, enforcing as best can be done, a common subspace.  These hidden estimates are later exploited in detecting damaged sensors on the basis of their similar operation, thus making the common subspace constraint important. To that end, we exploit an important property of commutation between the operators that are responsible for transforming the data into the common subspace. Linear operators $\bm{A} \in {\rm I\!R^{nxn}}$ and $\bm{B} \in {\rm I\!R^{nxn}}$ are said to commute if, 
	
\begin{equation}
	[\bm{A}, \bm{B}] = \bm{A}\bm{B} - \bm{B}\bm{A} = 0.
	\label{comm_def}
\end{equation}

\begin{theorem} \label{commutation1} \cite{Frobenius_Commutation,commoneigbasis}
	If $\bm{A} \in {\rm I\!R^{nxn}}$ and $\bm{B} \in {\rm I\!R^{nxn}}$ are commuting linear operators, they share common eigenvectors.
\end{theorem}
\begin{theorem}\label{commutation2} \cite{Frobenius_Commutation, commoneigbasis}
	If $\bm{A} \in {\rm I\!R^{nxn}}$ and $\bm{B} \in {\rm I\!R^{nxn}}$ are commuting operators that are also individually diagonalizable, they share a common eigenbasis.
\end{theorem}
If the operators $\bm{Z^1}, ..., \bm{Z^L}$ commute, they will share a common eigenbasis \cite{Frobenius_Commutation, commoneigbasis}. Furthermore, since the transformation, $\bm{Z^l}[\bm{M^l}(\bm{X^l})]$ lies in the range space of the operator, $\bm{Z^l}$, $\forall l, \bm{\hat{H}^l} = \bm{Z^l}[\bm{M^l}(\bm{X^l})]$ lie in a common subspace, due to the shared basis. While exact commutation cannot be guaranteed, we include a penalty term in the optimization to encourage the operators to commute, hence leading to operators that are `almost commuting'. Note that $\bm{Z^l}$ must be a square matrix for validly enforcing the commutation cost as per Equation \ref{comm_def}. By including this with the GAN loss, Equation \ref{WGAN_eqn} becomes, 
\begin{equation}
\label{GAN_1,comm}
\begin{split}
\min_{\bm{G^l}, \bm{Z^l}}\max_{\bm{D}}  \sum_{l=1}^{L} \{V(\bm{G^l},\bm{D}) + \gamma_2 \sum_{\substack{m=1\\m \neq l}}^{L}||\left[\bm{Z^l}, \bm{Z^m}\right]||^2\}.
\end{split}
\end{equation}
While the CGAN network is expected to eventually generate samples that lie in the same subspace as the target data, even in the absence of this auxiliary loss, we observe that its inclusion nevertheless improves the generator performance tremendously, in spite of the fact that the target space is ill defined. Experiments show that including the commutation term speeds up convergence, and also yields better hidden space estimates.
\subsection{Structure of the Hidden Space/Selection Operators}
In addition to ensuring that the hidden space estimates generated by different modalities lie in a common subspace, it is also important to structure the hidden space in a way that it not only contains information shared between the sensor modalities, but also contains private information of individual modalities. This is especially important for heterogeneous sensors as they may provide additional information along with some common information. This conforms with the notion that no two sensors are completely similar or dissimilar. This is equivalent to structuring $\bm{S^l}$ such that it selects features that are relevant to the $l^{th}$ modality, while ignoring others. Such a structure can be achieved by encouraging the columns of the selection matrix to be close to zero. If the $m^{th}$ column of $\bm{S^l}$ is zeroed out, then the $m^{th}$ feature in the hidden space, $\bm{H}$, does not contribute toward the $l^{th}$ modality. This allows us to find the latent space, $\bm{H}$, that naturally separates information shared between different modalities from which is private to each modality. We implement this by minimizing the $L_{\infty,1}$ norm, i.e., $\min_{\bm{S^l}} ||\bm{S}^{\bm{l}}||_{\infty,1}$, where,
\begin{equation}
	||\bm{S}^{\bm{l}}||_{\infty,1} = \sum_{j} \max_i (|s^{l}_{ij}|).
\end{equation}
This minimizes the sum of the maximum of each column in $\bm{S^l}$. Upon adding this term to the generator cost functional, we get,
\begin{equation}\label{GAN_1,comm,inf}
	\begin{split}
	\min_{\bm{G^l}, \bm{S^l}, \bm{Z^l}}\max_{\bm{D}}  \sum_{l=1}^{L} \{V(\bm{G^l},\bm{D}) + \gamma_1 ||\bm{S}^{\bm{l}}||_{\infty,1}\\ + \gamma_2 \sum_{\substack{m=1\\m \neq l}}^{L}||\left[\bm{Z^l}, \bm{Z^m}\right]||^2\}
	\end{split}	
\end{equation}

Notice that the above formulation has a trivial solution of setting $\bm{S^l} = \bm{0}$, and, $\bm{Z^l } = 0$ for all $l$. In order to avoid this solution, we also optimize the classification based on the selected features, $\bm{F^l}$, for each modality, via the minimization of the cross-entropy loss. This ensures that the learned features, $\bm{F^l}$, are optimal for object detection/recognition, which is only possible if $\bm{S^l} \neq 0$, and, $\bm{Z^l} \neq 0$.
Given the sensor observations, $\bm{X^l}$, and the classifier $\bm{C^l} = \{c_i^l\}$, the cross-entropy loss is computed as, 
\begin{equation}
C_{LOSS}^l(\bm{F^l}) = \sum_{n=1}^{N} \sum_{i=1}^{I} -y_{n_i} \log \sigma(c^l_i(\bm{f^l_n})),
\end{equation}
where, $Y_n = \{y_{n_i}\}$ is the ground truth for the $n^{th}$ sample, $\sigma$ is the soft-max function, and, $\bm{f_n^l} = \bm{S^l}[\bm{G^l}(\bm{x_n^l})]$.

Finally, the optimization task is,
\begin{gather}
\label{Final_Obj}
\begin{align}
&\nonumber \min_{\bm{G^l}, \bm{Z^l}, \bm{S^l}, \bm{C^l}}\max_{\bm{D}} \mathcal{L}(\bm{G^l}, \bm{D}, \bm{Z^l}, \bm{S^l}, \bm{C^l})\\
&\nonumber \mathcal{L}(\bm{G^l}, \bm{D}, \bm{Z^l}, \bm{S^l}, \bm{C^l}) = \sum_{l=1}^{L} \{V(\bm{G^l},\bm{D}) + \gamma_1 ||\bm{S^l}||_{\infty,1}\\ 
&\quad \quad \quad \quad \quad + \gamma_2 \sum_{\substack{m=1\\m \neq l}}^{L}||\left[\bm{Z^l}, \bm{Z^m}\right]||^2 + \gamma_3 C_{LOSS}^l({\bm{F^l}}), 
\end{align}
\end{gather}



where, $\gamma_1$, $\gamma_2$, and $\gamma_3$ are hyper-parameters that control the contribution of different terms toward the optimization.

The necessary updates to train these networks are summarized in the Appendix (Algorithm 1). 
This setup learns the optimal features for classification $\bm{F^l}$, while driven by the existence of the hidden space $\bm{H}$, such that $\bm{F^l} = \bm{S^lH}$. Due to the inclusion of the selection matrix $\bm{S^l}$, followed by the Classification Layer, note that hand-crafting of features as in \cite{rohedaEDF, rohedaEDF-jour}, is no longer required. The features of interest are learned and automatically selected by the optimization during the training phase. 
The effects of the different terms in Equation \ref{Final_Obj} on the determined hidden spaces are shown in Section \ref{Perf_Analysis} in Figures \ref{GAN_only}, \ref{GAN+Comm}, and \ref{GAN+Comm+Sel}. The block diagram for the proposed approach is summarized in Figure \ref{FUSIONGANsBD}.

\subsection{A Special Case of Event Driven Fusion}
\label{Sp_EDF}
In order to fuse the individual decisions from the sensors we use a special case of Event Driven Fusion. Instead of defining feature events for each object as in \cite{rohedaEDF}, an event in this case is the occurrence of the $i^{th}$ object as seen by the $l^{th}$ sensor, $o_i^l$. The $l^{th}$ classifier, $\bm{C^l} = \{c_i^l\}$, provides the corresponding probability given the test sample, $\bm{x_t^l}$, $P_t^l(o_i^l) = \frac{exp(c_i^l(\bm{f_t^l}))}{\sum_{m=1}^{I} exp(c_m^l(\bm{f_t^l}))}$, as previously discussed in Section \ref{problemform}. Each sensor report is now represented as, $R_t^l = \{P_t^l(o_i^l)\}$, and the fused probability of occurrence of the $i^{th}$ object as per the rules of Event Driven Fusion \cite{rohedaEDF} is determined as,
\begin{equation}
\begin{split} 
P_t^f(o_i) = P_t(o_i^1,o_i^2,...,o_i^L) = \rho.P_{\text{MAX MI}}(o_i^1,o_i^2,...,o_i^L)\\ + (1-\rho).P_{\text{MIN MI}}(o_i^1,o_i^2,...,o_i^L),
\end{split}
\end{equation} 
where $\rho$ is a pseudo-measure of correlation between the sensor modalities. For making an informed decision in favor of an object, this formulation assumes all the modalities to be equally reliable. This is not always true in practice, as certain sensors may provide more discriminative information than others, and are hence more reliable. It is thus important to weigh the various sensor decisions by a Degree of Confidence (DoC). The DoC in the decisions made by the $l^{th}$ sensor given the test sample, $\bm{x_t^l}$, is denoted by, $DoC_t^l \in [0,1]$, where $DoC_t^l = 0$ implies that the sensor observations do not provide any useful information about the target identity, and $DoC_t^l = 1$ implies that information provided by the sensor is highly discriminative with respect to target classification. The individual sensor reports, $R_t^l = \{P_t^l(o_i^l)\}$, are now redefined as,
\begin{equation}
R_t^{l'} = \{P_t^{l'}(o_1^l), P_t^{l'}(o_2^l),..., P_t^{l'}(o_I^l), P_t^{l'}(unc^l)\},
\end{equation}
where $P_t^{l'}(o_i^l) = DoC_t^l*P_t^l(o_i^l)$, and $P_t^{l'}(unc^l) = 1 - DoC_t^l$ is the probability that the $l^{th}$ sensor is uncertain about the target identity. The joint probability distribution for the new sensor reports is now rewritten as,
\begin{equation}
	\begin{split}
	 P_t(a^1,a^2,...,a^L) = \rho.P_{\text{MAX MI}}(a^1,a^2,...,a^L)\\ + (1-\rho).P_{\text{MIN MI}}(a^1,a^2,...,a^L),
	\end{split}
\end{equation}
where, $a^l \in \{o_1^l,o_2^l...,o_I^l,unc^l\}$, and the probability of occurrence of the $i^{th}$ target is computed as,
\begin{equation}
\begin{split}
P_t^f(o_i) = P_t((\bigwedge_{l=1}^{L} o_i^l) \bigvee_{m_1 = 1}^{L} (unc^{m_1} \bigwedge_{\substack{l=1 \\ l \neq m_1}}^L o_i^l) \\ \bigvee_{m_1,m_2 = 1}^{L} (unc^{m_1} \wedge unc^{m_2} \bigwedge_{\substack{l=1 \\ l \neq m_1 \\ l \neq m_2}}^L o_i^l) \\ ... \bigvee_{\substack{m_1,m_2\\...m_{L-1} = 1}}^L (\bigwedge_{j=1}^{L-1} unc^{m_j} \bigwedge_{\substack{l = 1 \\ \forall j, l \neq m_j}}^L o_i^l)),
\end{split}
\end{equation}
with a degree of confidence, $DoC_t^f = 1 - P_t(\bigwedge_{l=1}^L unc^l)$, in the fused decision. 

For a 2D scenario (i.e. two sensors), the probability of occurrence of the $i^{th}$ target is, 
\begin{equation}
P_t^f(o_i) = P_t((o_i^1 \wedge o_i^2) \vee (unc^1 \wedge o_i^2) \vee (o_i^1 \wedge unc^2)),
\end{equation}
with a degree of confidence, $DoC^f = 1 - P(unc^1 \wedge unc^2)$, in the fused decision. This formulation takes into account the potential uncertainties in the decisions made by individual sensors. Figure \ref{JD_2D} shows the joint distribution to be determined in a 2D case.

\begin{figure}
\centering
\includegraphics[width=0.4\textwidth]{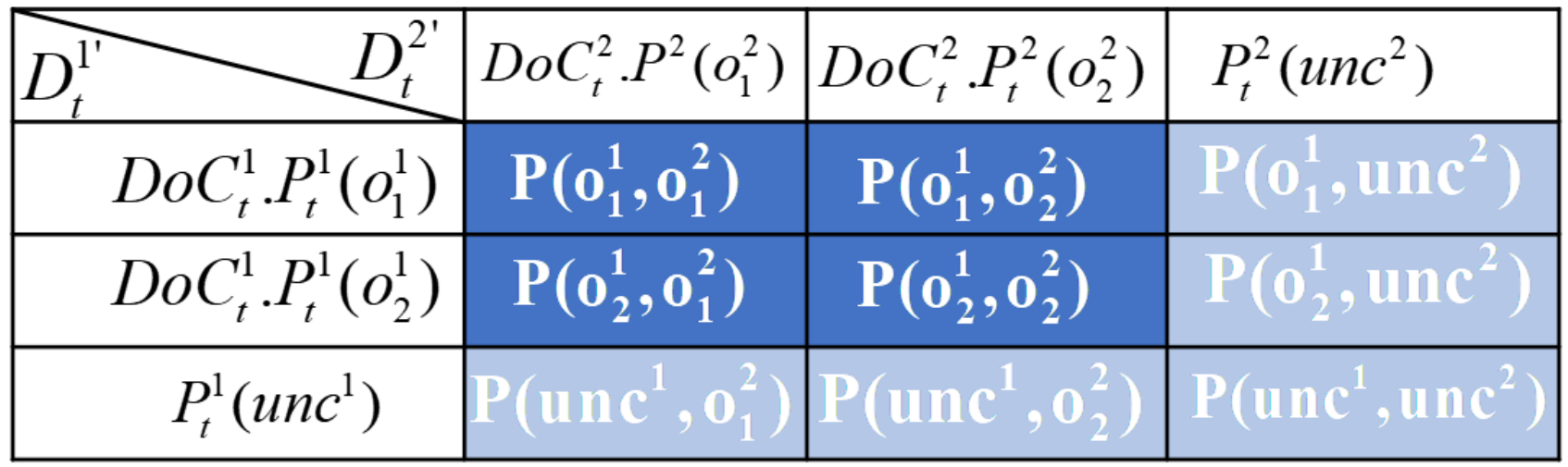}
	\caption{Fused joint probability distributions when accounting for the uncertain event. The joint distribution is determined following the rules of Event Driven Fusion.}
\label{JD_2D}
\end{figure}

The degree of confidence in the $l^{th}$ sensor for classification of the test sample, $\bm{x_t^l}$, is determined as, 
\begin{equation}
DoC^l_t = (1-p_D(\bm{\hat{h}_t^l})).Acc_{train}^l,
\label{Adaptive_DoC}
\end{equation}
where, $p_D(\bm{\hat{h}_t^l})$ is the probability of damage for the $l^{th}$ sensor given the test sample, $\bm{x_t^l}$ (discussed in detail in the next subsection), and $Acc_{train}^l$ is the training accuracy of the $l^{th}$ sensor. The training accuracy in Equation \ref{Adaptive_DoC} represents prior information about the discrimination power of the sensor, while $(1 - p_D(\bm{\hat{h}_t^l}))$ represents the sensor condition at the time of operation/testing.

\subsection{Sensor Failure Detection}
\label{SenorFailureDet}
Since $\bm{H}$ is designed so that it be common for all the sensors, it may be used to detect a damaged sensor.
\subsubsection{Cross-Sensor Tracking}
The estimate $\bm{\hat{H}^m}$, based on erroneous observations, will significantly deviate from the estimates from normal observations (i.e. sensors that are not damaged), $\bm{\hat{H}^l}, l \neq m$. This allows detection of damage in a sensor during the testing phase. The $m^{th}$ sensor is said to be damaged if,
\begin{gather}
\begin{align}
&\nonumber \quad \quad \text{ } \forall j,l \neq m,\\
&\nonumber \quad \quad \text{ } \sum_{\substack{l=1\\ l \neq m}}^{L} \mathcal{I}(||\bm{\hat{h}_t^m} - \bm{\hat{h}_t^l}||^2 > T) \geq \begin{cases}
\frac{L-1}{2},\text{ if } L \text{ is odd},\\
\frac{L}{2} - 1,\text{ if } L \text{ is even},
\end{cases}\\
&\text{And,  } \sum_{\substack{j,l=1\\ j,l \neq m}}^{L} \mathcal{I}(||\bm{\hat{h}_t^j} - \bm{\hat{h}_t^l}||^2 < T) \geq \begin{cases}
\frac{L-1}{2},\text{ if } L \text{ is odd},\\
\frac{L}{2} - 1,\text{ if } L \text{ is even},
\end{cases}	 
\end{align} 
\end{gather}
where, $\mathcal{I}(.)$ is the indicator function, $T$ is a threshold value determined from the training data, and $\bm{\hat{h}_t^l}$ is the estimated hidden space given the observation $\bm{x_t^l}$. This equation considers a sensor damaged when the distance between the hidden spaces of more than half the sensors in the network is less than some threshold $T$, while the distance of the hidden space of the sensor in question is more than $T$. 
The limitation with this approach is that we can only detect up to $\frac{L-1}{2} / \frac{L}{2} - 1$ damaged sensors. 
\subsubsection{Hierarchical Clustering}
A faulty sensor can also be detected by comparing the hidden space generated by the test sample with that generated from training data. One way to proceed is by clustering functional sensor observations (from training data), and verify whether the hidden space estimate generated by the test sample can be associated to any of these clusters. The concatenation of the training data, $\hat{\bm{H}}_{d \times LN}^C = \{ \hat{\bm{H}}^1, \hat{\bm{H}}^2, ..., \hat{\bm{H}}^L \}$, is first used to construct a clustering tree based on an Agglomerative approach (see Algorithm 2 in Appendix). The probability of damage of the $l^{th}$ sensor is then computed as,
\begin{equation}
	p_D(\bm{\hat{h}_t^l}) = \frac{d_{lev}}{\max_v d_v},
\end{equation}
where, $d_v$ is the cut-off distance at clustering level $v$, and,  
\begin{gather}
	\nonumber lev = arg\min v\\
	s.t. \exists j \in \{1,...,J_{lev}\}, \text{ } \bm{\hat{h}_t^l} \in Z_{lev}^j,
\end{gather}
$v =1,...,V$ are the clustering levels, $J_v$ is the number of clusters at level $v$, and $Z_v^j$ is the $j^{th}$ cluster at level $v$. This is a measure of how quickly the hidden space estimate, $\bm{\hat{h_t^l}}$, can be clustered with the training data.
 
The $l^{th}$ sensor is said to be damaged if, 
\begin{equation}
p_D(\bm{\hat{h}^l_t}) > T,
\label{Treshold_eq} 
\end{equation}
where, $T$, is a threshold value which will depend on the dataset, types of sensors, SNR, etc. In our evaluation, we compute the optimal thresholds at different SNRs for the training data, and these thresholds are later used during the testing phase in order to determine the state of a sensor. This probability measure is also used in Equation \ref{Adaptive_DoC}, in order to adapt the DoC based on the sensor condition in a functional mode.
If the sensor is damaged, the representative features are generated using the selection matrix as,
\begin{equation}
	\bm{\hat{f}_t^l} = \bm{S^l}\left(\frac{\sum_{m \in \Gamma} DoC_t^m.\bm{\hat{h}_t^m}}{\sum_{m \in \Gamma} DoC_t^m}\right),
	\label{Reconstruction}
\end{equation}
where, $\Gamma$ is the set of working sensors.

\begin{figure}
	\centering
	\includegraphics[width=0.45\textwidth]{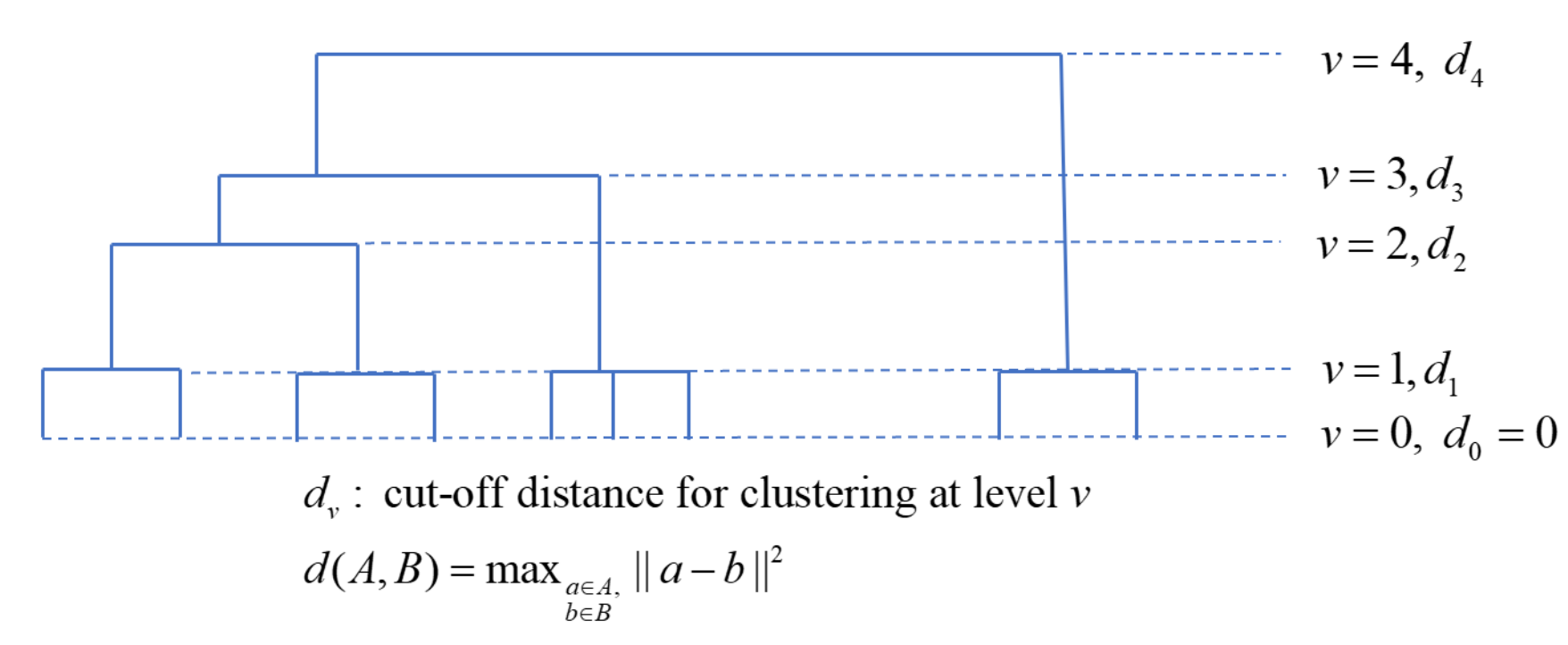}
	\caption{A clustering tree created using Hierarchical Clustering (Agglomerative Clustering)}
\end{figure}

\section{Experiments and Results}
We validate our proposed approach by running experiments on two different datasets.
\begin{figure}[h!]
	\centering
	\subfloat[]{\includegraphics[width=0.45\textwidth]{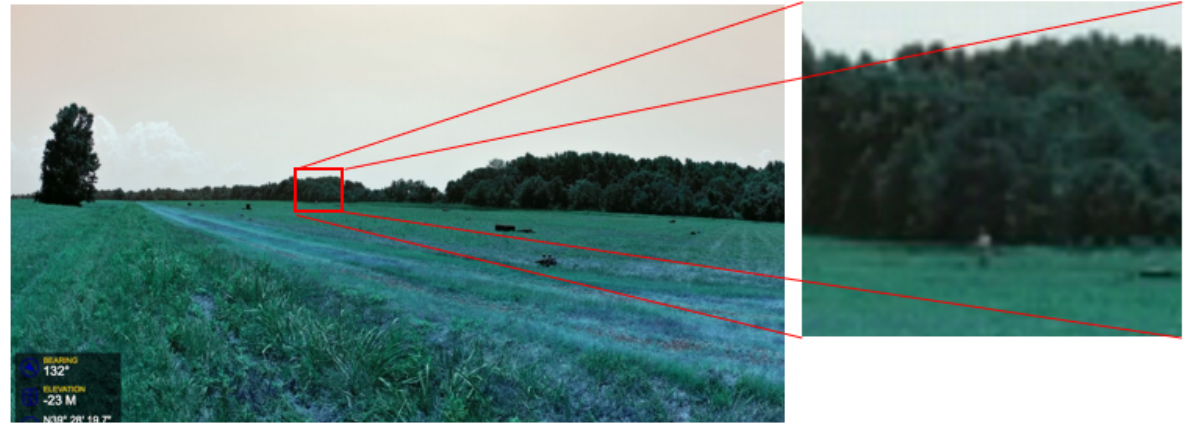}}\hfill
	\subfloat[]{\includegraphics[width=0.4\textwidth]{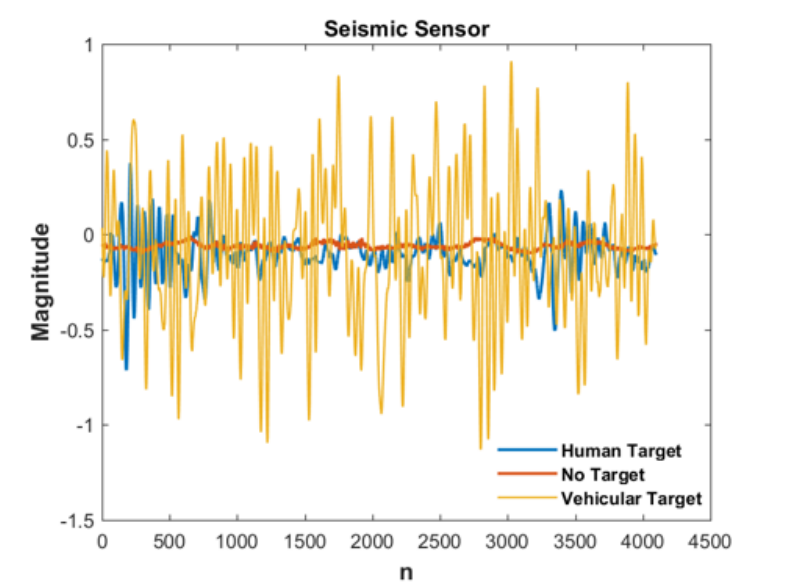}}\hfill
	\subfloat[]{\includegraphics[width=0.4\textwidth]{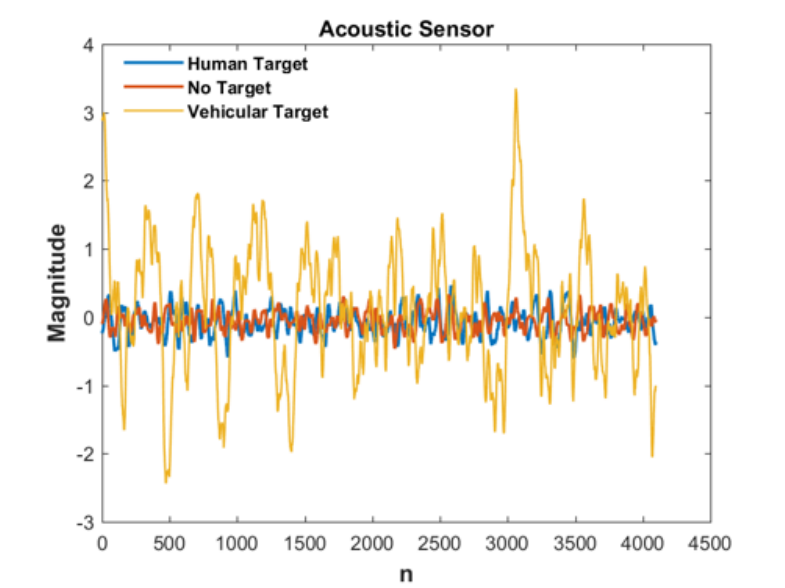}}\hfill
	\caption{(a): Sample video frame with a human target, (b): Sample seismic sensor observations for human target, vehicular target, and no target cases, (c): Sample acoustic sensor observations for human target, vehicular target, and no target cases.}
	\label{data_samplesARL}
\end{figure}

\begin{figure}[h!]
	\centering
	\includegraphics[width=0.4\textwidth]{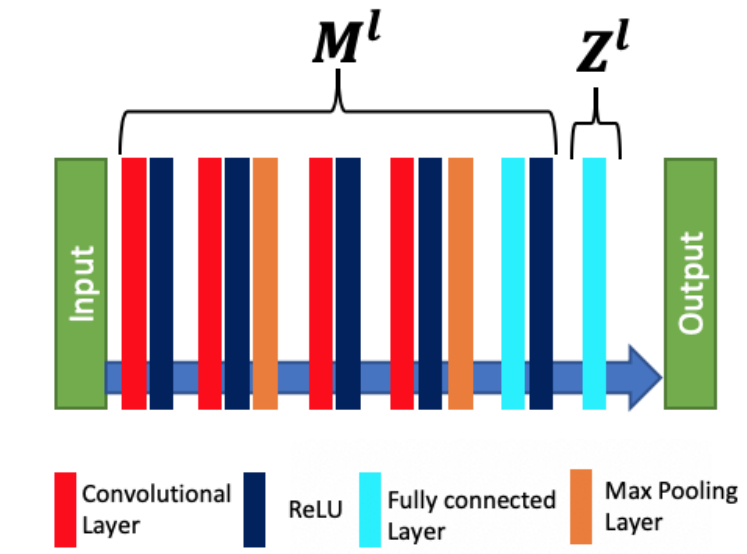}
	\caption{Generator Network used for each modality}
	\label{Generator_Structure}
\end{figure}

\begin{figure*}[tbp]
	\centering
	\subfloat[Dataset-1]{\includegraphics[width = 0.48\textwidth,height=0.18\textheight]{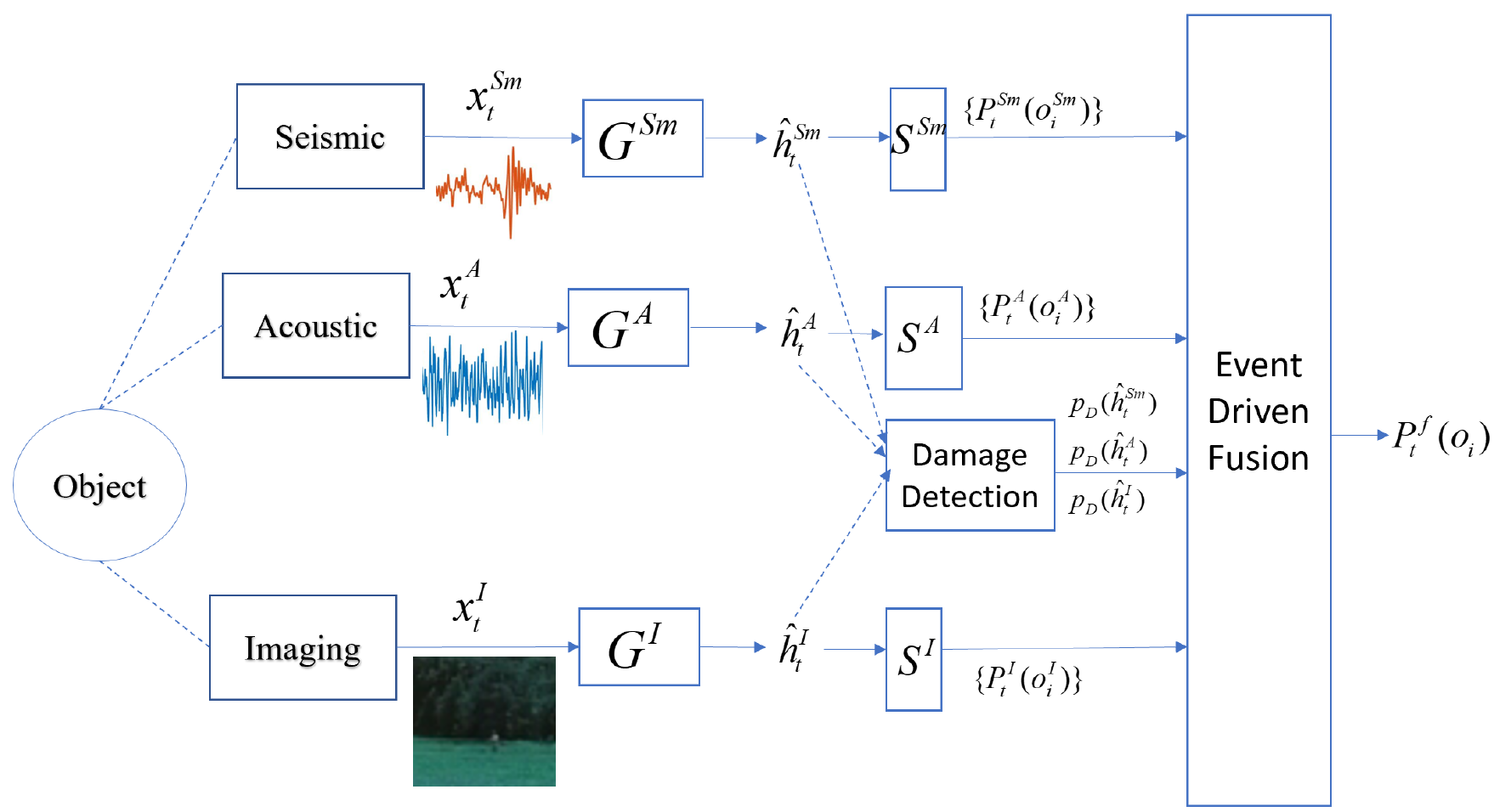}}\hfill
	\subfloat[Dataset-2]{\includegraphics[width=0.48\textwidth, height=0.18\textheight]{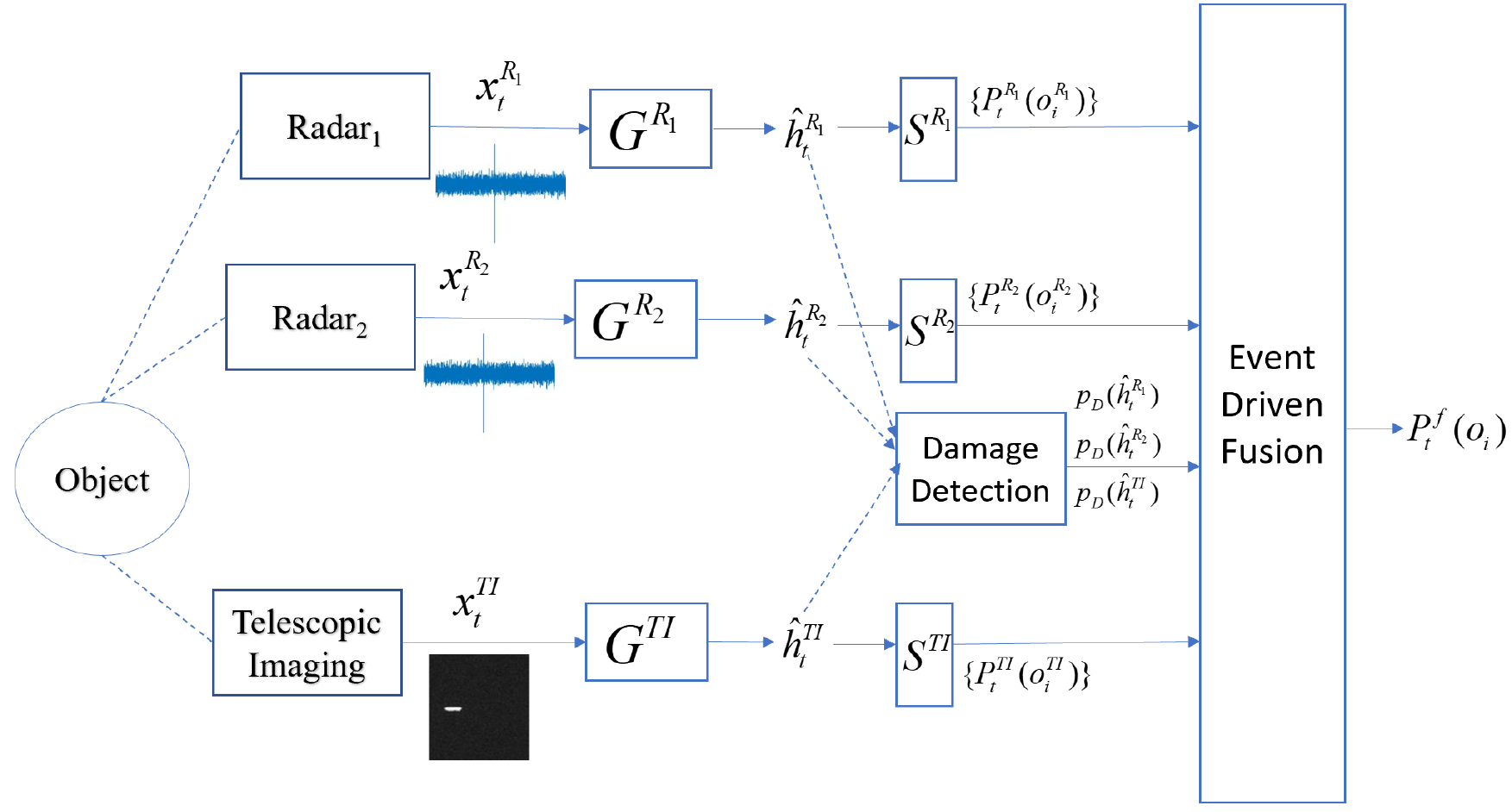}}
	\caption{Block Diagram of Implementation}
	\label{BD}
\end{figure*}
The first dataset we use is pre-collected data from a network of seismic, acoustic, and imaging sensors deployed in a field, where people/vehicles were walking/driven around in specified patterns. Details about this sensor setup and experiments can be found in \cite{nabritt2015personnel}. This dataset has been previously used for target detection in \cite{rohedaEDF, roheda2018cross, lee2017accumulative, ARL_paper_SA}. Here, we use this dataset to classify between human targets, vehicular targets, and no targets. Some data samples from the sensors can be seen in Figure \ref{data_samplesARL}. This will be referred to as `Dataset-1' in the following discussions.

For our second experiment, we select two sensors, namely a Radar sensor and a telescopic optical sensor, for one (latter) of which we have acquired real data. For technical reasons, our Radar measurements were never co-measured with the optical data, and so we use MATLAB Simulink in order to simulate the radar responses. Both sensors are ideally synchronized when observing a given target, which in our case, can be any object in outer space, such as satellite or space debris. 
Each generated radar signal over one second is correlated with two telescopic images. Samples for objects with different velocities, cross-sections, ranges, and aspect-ratios are then generated. 

Object classes are defined in the same way as in \cite{rohedaEDF-jour}, based on events on the feature values. Note that these events are no longer required for training purposes, and are only used for the purposes of labeling the data.
For the radar, we use the features, velocity ($v$), cross-section ($cs$), and range ($r$), and the events are defined as,
\begin{gather}
\nonumber a_1^v: 0 \leq v \leq 10 \text{ }m/s,\text{ } a_2^v: 15\text{ }m/s \leq v \leq 35\text{ }m/s,\\
\nonumber a_1^r: 0 < r \leq 300 \text{ }m,\text{ }a_2^r: 300\text{ }m < r, \\
a_1^{cs}: 0 < cs \leq 20 \text{ }m^2, \text{ }a_2^{cs}: 15\text{ }m^2 \leq cs \leq 50\text{ }m^2.
\end{gather}
Note that the events $a_i^v, a_i^r,$ and $a_i^{cs}$ are defined in the same way as $a_{kj}^l$ in Section \ref{EDF}. 
From the telescopic imaging sensor, the features, displacement ($d$) and aspect ratio ($AR$) define the following events,
\begin{gather}
\nonumber a_1^d: 0 \leq d \leq 60 \text{ }pixels,\text{ }a_2^d: 90\text{ }pixels \leq d \leq 210\text{ }pixels,\\
a_1^{ar}: 0 < AR \leq 1.5, \text{ }a_2^{ar}: 1.5 < AR.
\end{gather}
Furthermore, the objects for classification are defined in terms of these events as,
\begin{gather}
o_1: \{a_1^r \wedge [(a_2^v \wedge a_2^d) \vee (a_2^{cs} \vee a_2^{ar})]\}\\
o_2: \{a_1^v \wedge a_1^d \wedge a_2^r \wedge a_1^{cs} \wedge a_1^{ar}\}
\end{gather} 
This will be referred to as `Dataset-2' in the following discussions. 

\subsection{Implementation Details}
Figure \ref{BD}-(a) shows the detailed block diagram for implementation of the discussed approach on Dataset-1. As noted earlier we no longer require handcrafted features and event definitions as in \cite{rohedaEDF, rohedaEDF-jour}, since we let the generative structure (see Figure \ref{BD}) guide the learning of features. Similarly, the block diagram for implementation of this system on Dataset-2 can be seen in Figure \ref{BD}-(b). 
The output of the $l^{th}$ generator, $\bm{G^l}$, for a test sample, $\bm{x_t^l}$, is a $d_H$-dimensional estimate of the hidden space, $\bm{\hat{h}^l_t}$. This hidden space is also used to detect potential damages to the sensors deployed. The optimal features $\bm{f_t^l}$, are subsequently selected by each sensor for making decisions on target identities via the selection matrix $\bm{S^l}$. The decisions are then fed into the fusion system which synthesizes a decision on the basis of the rules discussed in Section \ref{Sp_EDF}.

The generator network uses 1-D convolutions in case of seismic/acoustic and radar modalities, and 2-D convolutions in case of the imaging sensors. We use a 6 layered Neural Network, with 2x2 max pooling layer after the $2^{nd}$ and $4^{th}$ layer. ReLU activation is applied after every convolutional layer. The first 4 layers are convolutional, while the last two are fully connected. The first two convolutional layers use a filter size of 5 and the next two use a filter size of 3. 
The first fully connected layer is used to transform the output of the convolutional layers into a $d_H$ dimensional representation. All the layers preceding the final fully connected layer approximate the mapping $\bm{M^l}$, while the final layer approximates the operator $\bm{Z^l}$, and transforms the data into a common subspace.  


The discriminator is implemented as a 3 layered fully connected network. Each layer preceeding the final layer transforms the $d-$dimensional input to $\frac{d}{2}$ and applies ReLU activation. The final layer transforms the features to a $L-$dimensional vector which represents the predictions from the discriminator network.

In determining the dimension of the hidden space, we search over values ranging from $d=50$ to $d=5000$. For Dataset-1 we find the best performance at $d=500$, while for Dataset-2, $d=700$. In both cases we observe that at best performance, $d << d_l$.A search over $[10^{-5}, 10]$ was performed to select the optimal $\gamma_1$, $\gamma_2$, and $\gamma_3$. For Dataset-1 the set of parameters used were $\gamma_1= 10^{-3}, \gamma_2=10^{-4}, \gamma_3=1$ and $\gamma_1=10^{-3}, \gamma_2=10^{-3}, \gamma_3=1$ for Dataset-2. As mentioned in Section \ref{SenorFailureDet}, the value of the threshold $T$ will depend on the SNR of the signal as well as the sensor type. The optimal threshold for various sensors is computed based on the validation set during training. The SNR vs Threshold curve for the various sensors is shown in Fig. \ref{Thresholds}. 
\begin{figure}[h!]
	\centering
	\subfloat[]{\includegraphics[width=0.4\textwidth, height=0.22\textheight]{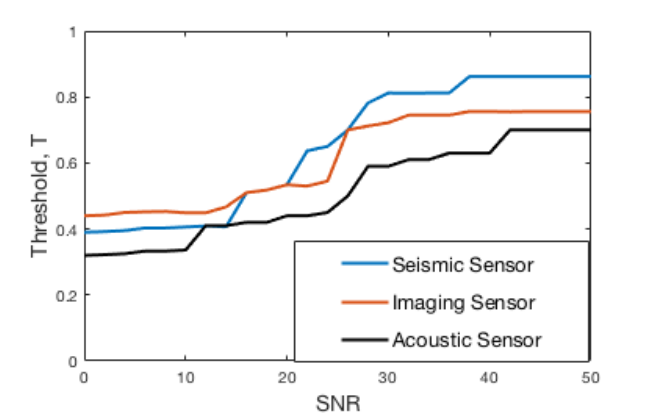}}\hfill
	\subfloat[]{\includegraphics[width=0.4\textwidth, height=0.22\textheight]{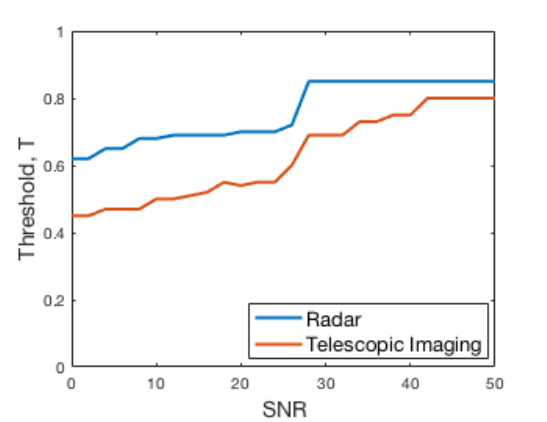}}\hfill
	\caption{Selected Thresholds for various sensors in (a): Dataset-1, (b): Dataset-2.}
	\label{Thresholds}
\end{figure}

\subsection{Performance Analysis}
\label{Perf_Analysis}
\begin{table}[htbp]
	\caption{Performance Comparison for both Datasets}
	\label{acctable_ARL}
	\begin{center}
		\begin{tabular}{|M{4.3cm}|M{1.7cm}|M{1.7cm}|}
			\hline
			\textbf{Method} & \textbf{Accuracy (Dataset-1)} &  \textbf{Accuracy (Dataset-2)}\\
			\hline
			Seismic Sensor & 93.62 \% & -\\
			\hline
			Acoustic Sensor &  68.71 \% & -\\
			\hline
			Imaging Sensor &  90.33 \% & - \\
			\hline
			Radar Sensor 1 & - & 89.33 \%\\
			\hline
			Radar Sensor 2 &  - & 86.73 \%\\
			\hline
			Telescopic Imaging Sensor &  - & 83.57 \%\\
			\hline
			Feature Concatenation &  88.13 \% & 86.45 \%\\
			\hline
			Dissimilar Sensor Fusion  &  91.61 \% & -\\ 
			\hline
			Similar + Dissimilar Sensor Fusion & - & 89.63 \%\\
			\hline
			Dempster-Shafer Fusion & 88.77 \% & 87.30 \%\\
			\hline
			Event Driven Fusion (Without using GAN structure) \cite{rohedaEDF, rohedaEDF-jour} & 92.04 \% & 90.36 \%\\
			\hline 
			Canonical Correlation Analysis \cite{ARL_paper_SA} & 86.64 \% & 85.69 \%\\
			\hline
			Discriminant Analysis & 91.33 \% & 88.21 \%\\
			\hline
			Dictionary Learning \cite{DL_paper} & 94.46 \% & 93.77 \% \\
			\hline
			Hidden Space Generated by GAN & 95.79 \% & 91.13 \%\\
			\hline
			\textbf{Event Driven Fusion + GAN} & \textbf{97.79 \%} & \textbf{95.34 \%}\\
			\hline
		\end{tabular}
	\end{center}
\end{table}
\begin{figure*}[tbp]
	\centering
	\subfloat[]{\includegraphics[width=0.33\textwidth]{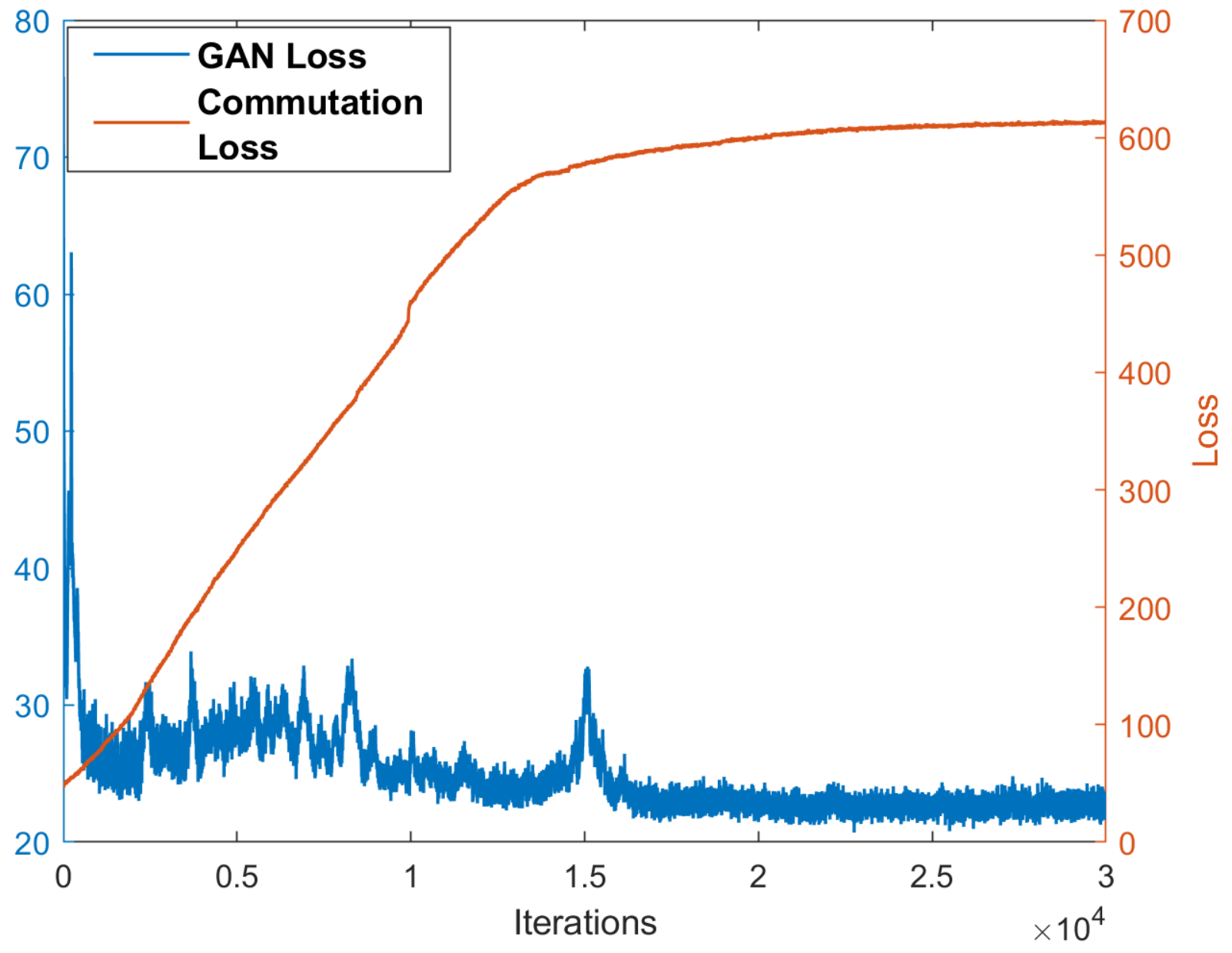}}\hfill
	\subfloat[]{\includegraphics[width=0.33\textwidth]{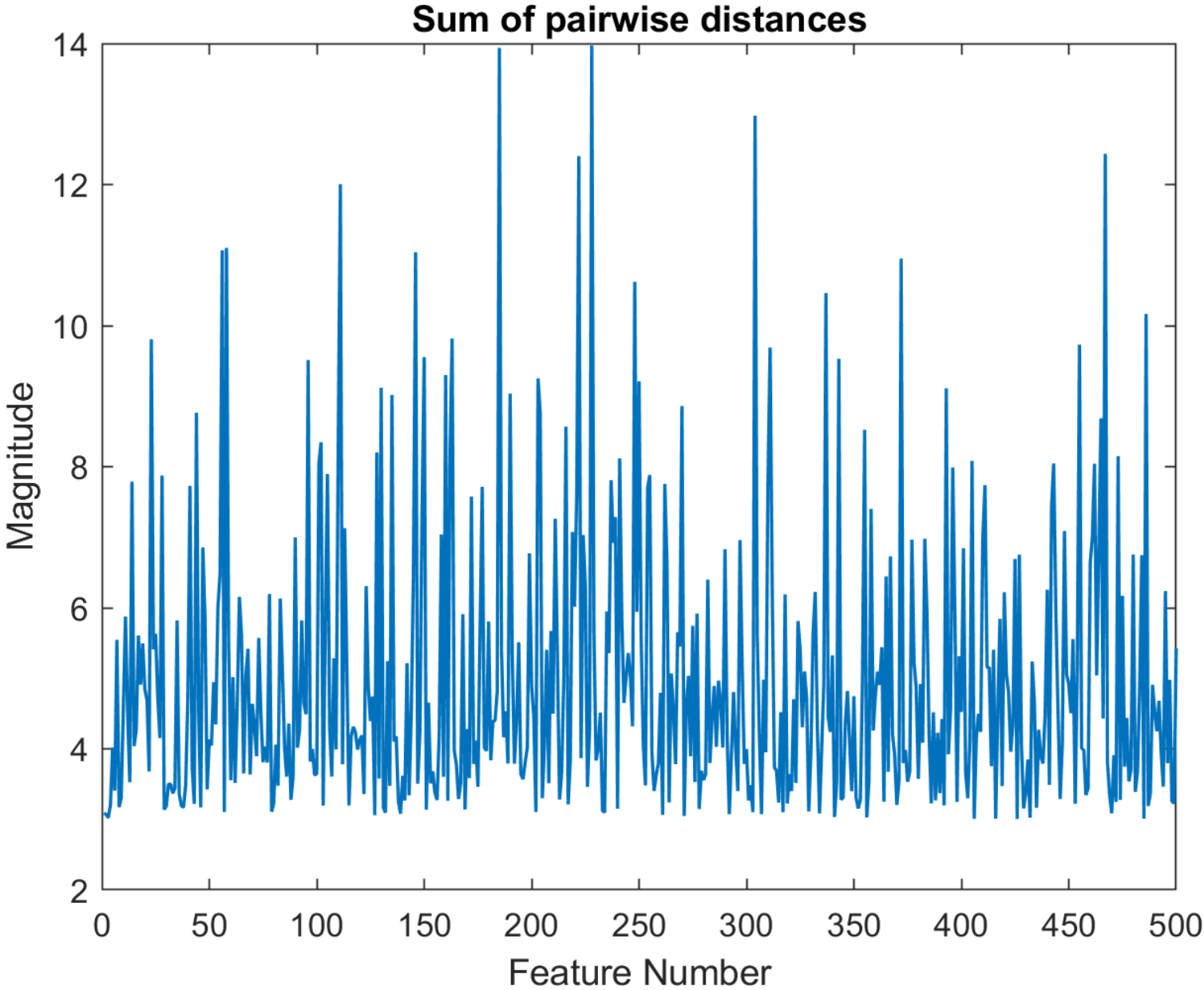}}\hfill
	\subfloat[]{\includegraphics[width=0.33\textwidth, height=0.2\textheight]{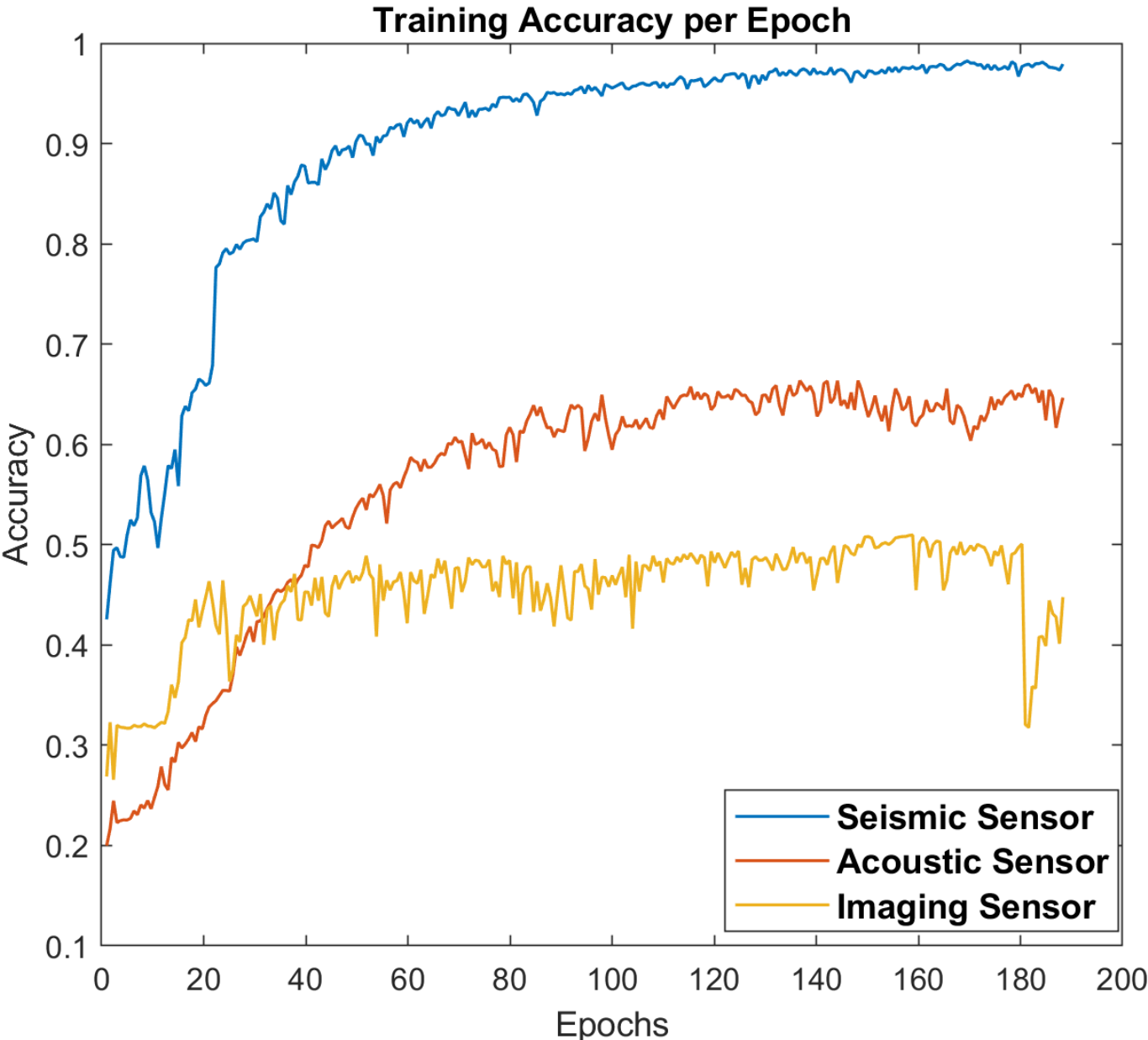}}\hfill
	\caption{(a): Optimization Losses, (b): Sum of pairwise distances between hidden estimates, (c): Training Accuracies/Epoch  when GAN Loss is optimized in tandem with classification loss}
	\label{GAN_only}
\end{figure*}
\begin{figure*}[tbp]
	\centering
	\subfloat[]{\includegraphics[width=0.33\textwidth]{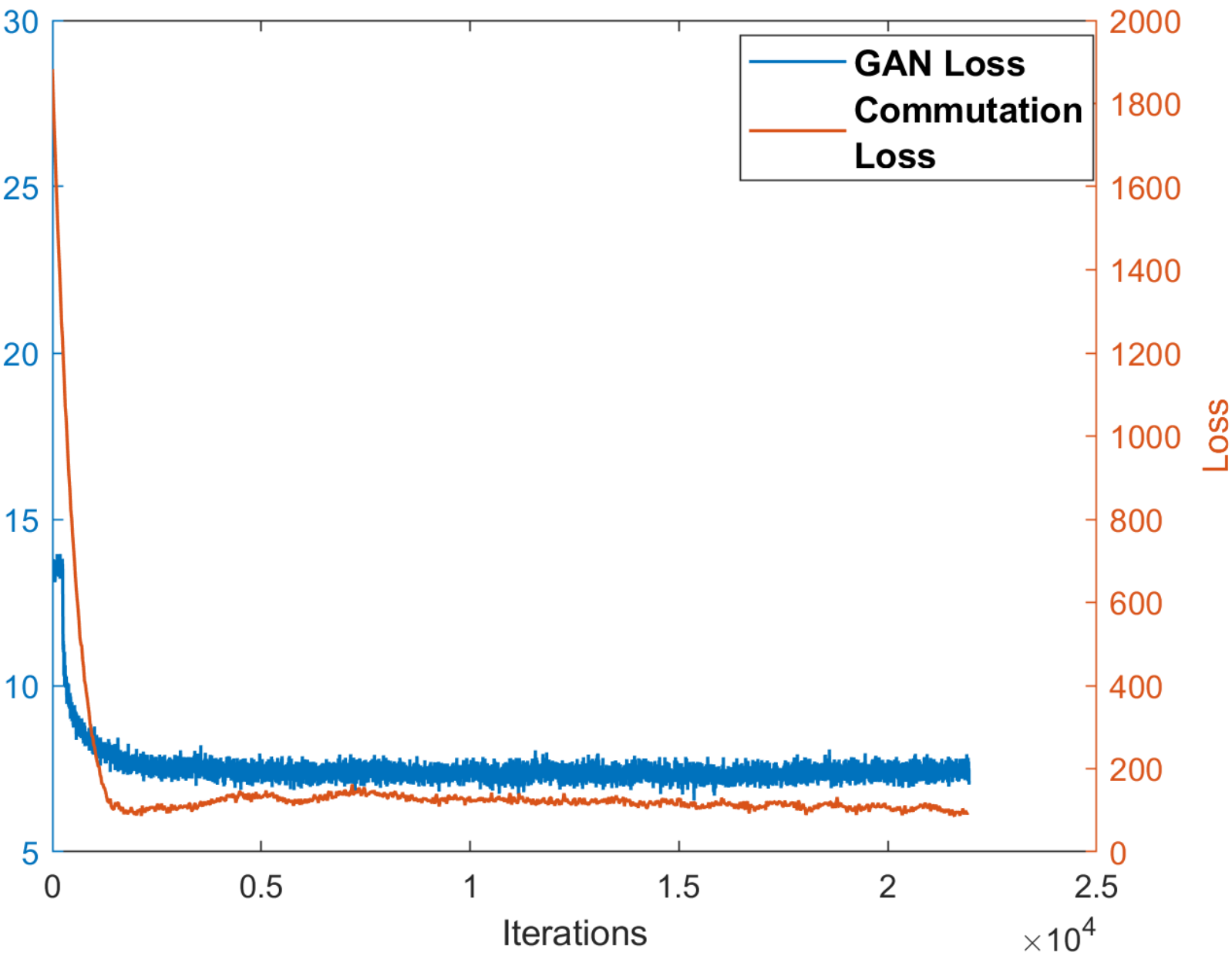}}\hfill
	\subfloat[]{\includegraphics[width=0.33\textwidth]{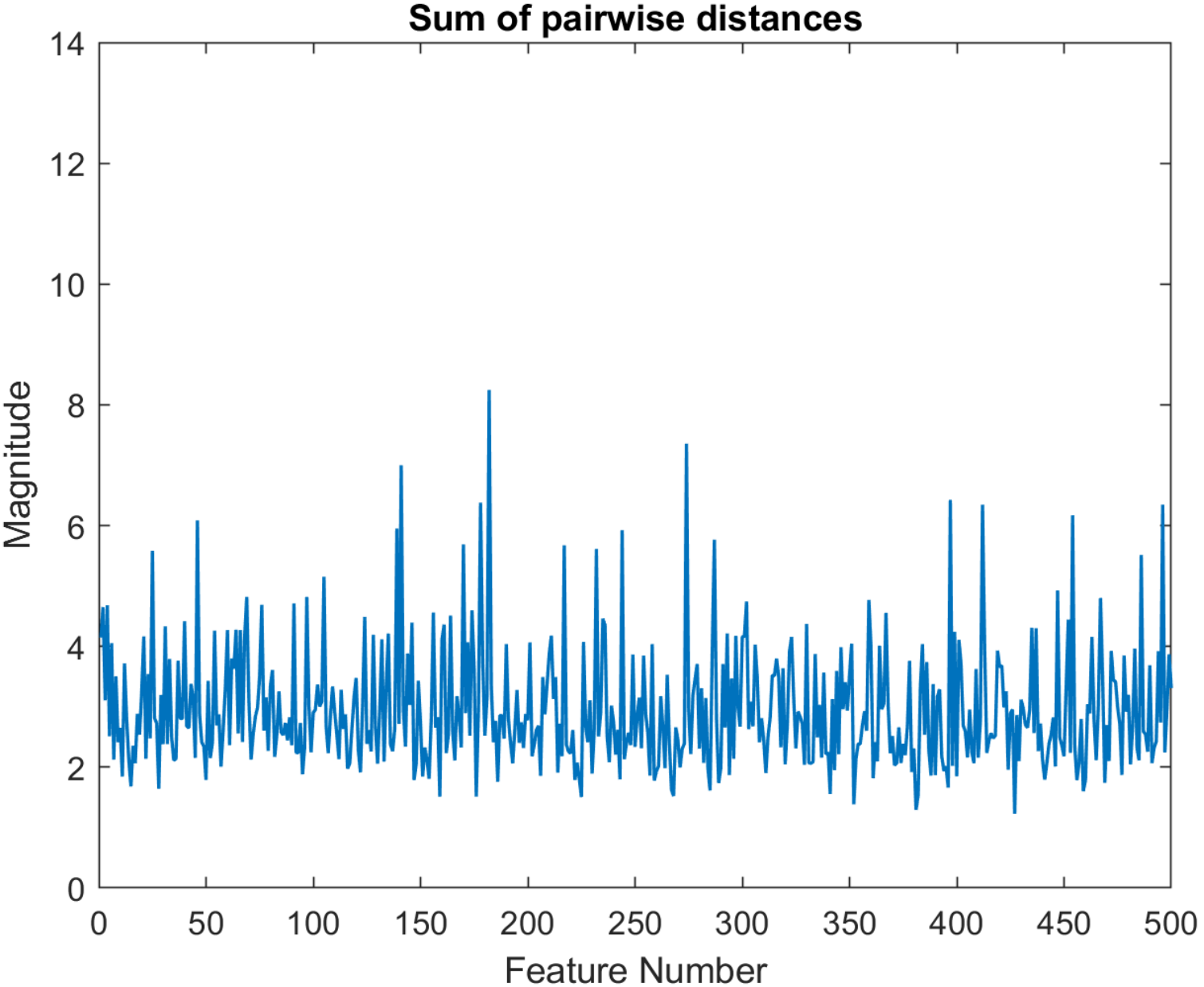}}\hfill
		\subfloat[]{\includegraphics[width=0.33\textwidth, height=0.2\textheight]{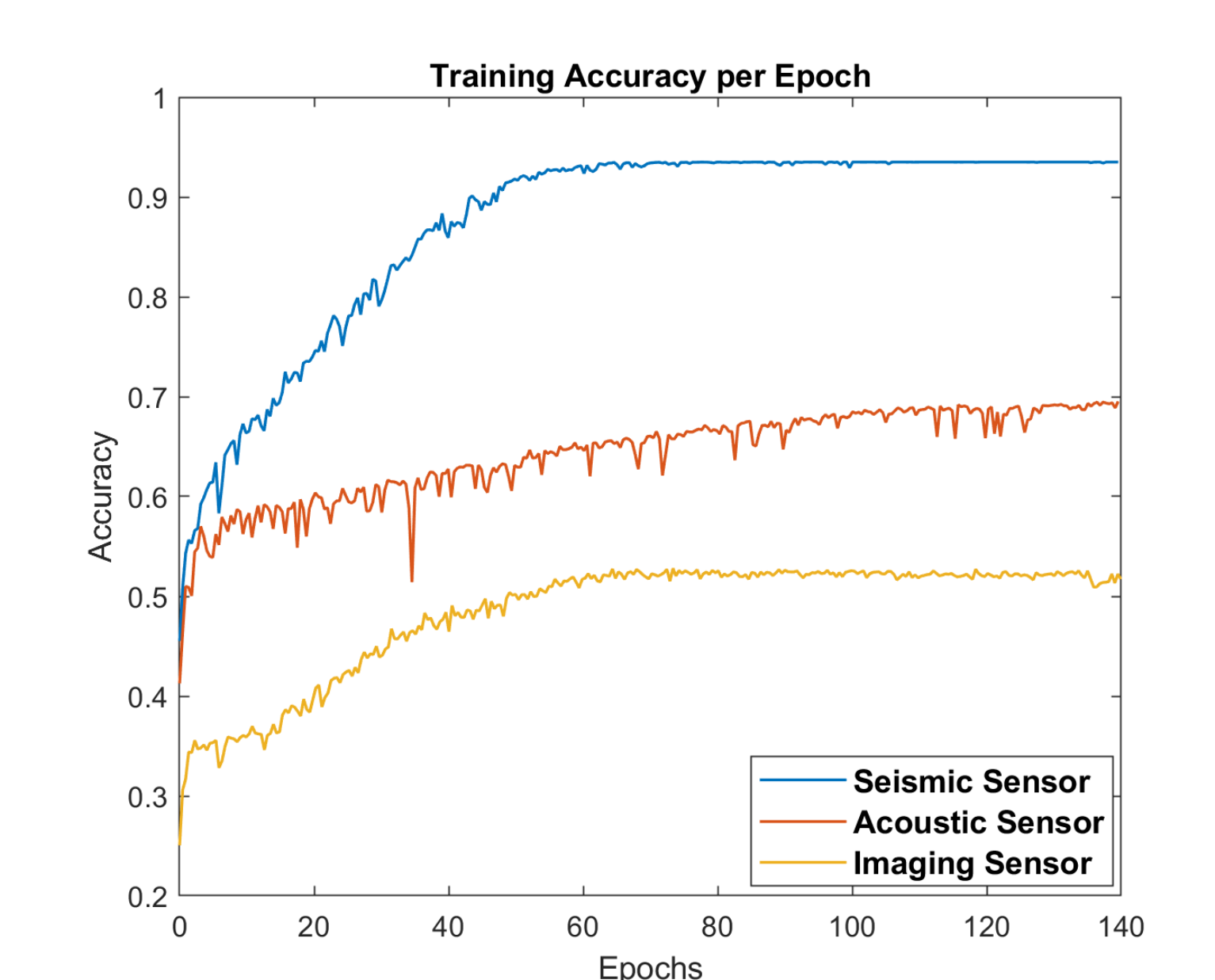}}\hfill
	\caption{(a): Optimization Losses, (b): Sum of pairwise distances between hidden estimates, (c): Training Accuracies/Epoch when commutation term is included with the GAN and classification term}
	\label{GAN+Comm}
\end{figure*}
\begin{figure*}[tbp]
	\centering
	\subfloat[]{\includegraphics[width=0.33\textwidth]{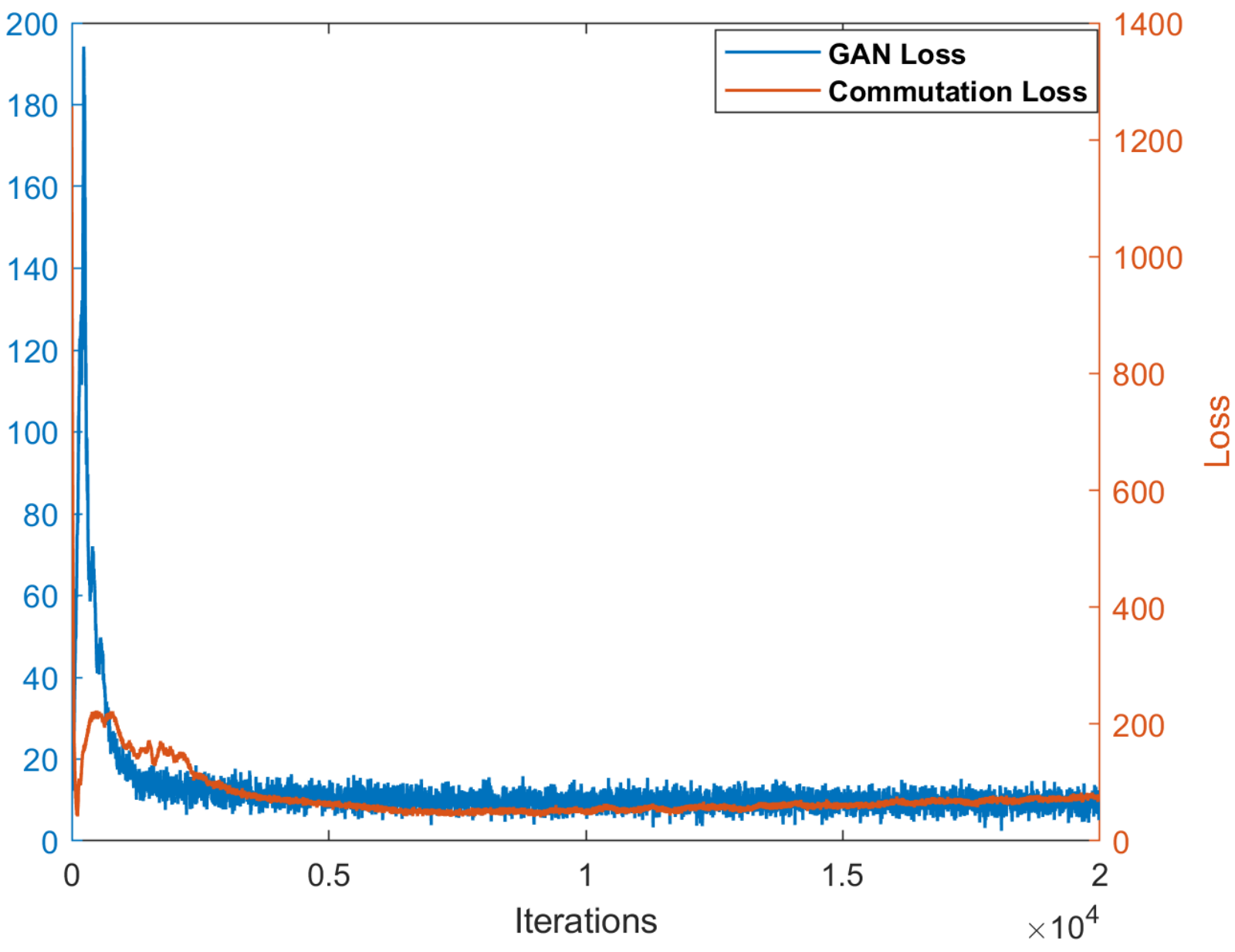}}\hfill
	\subfloat[]{\includegraphics[width=0.33\textwidth]{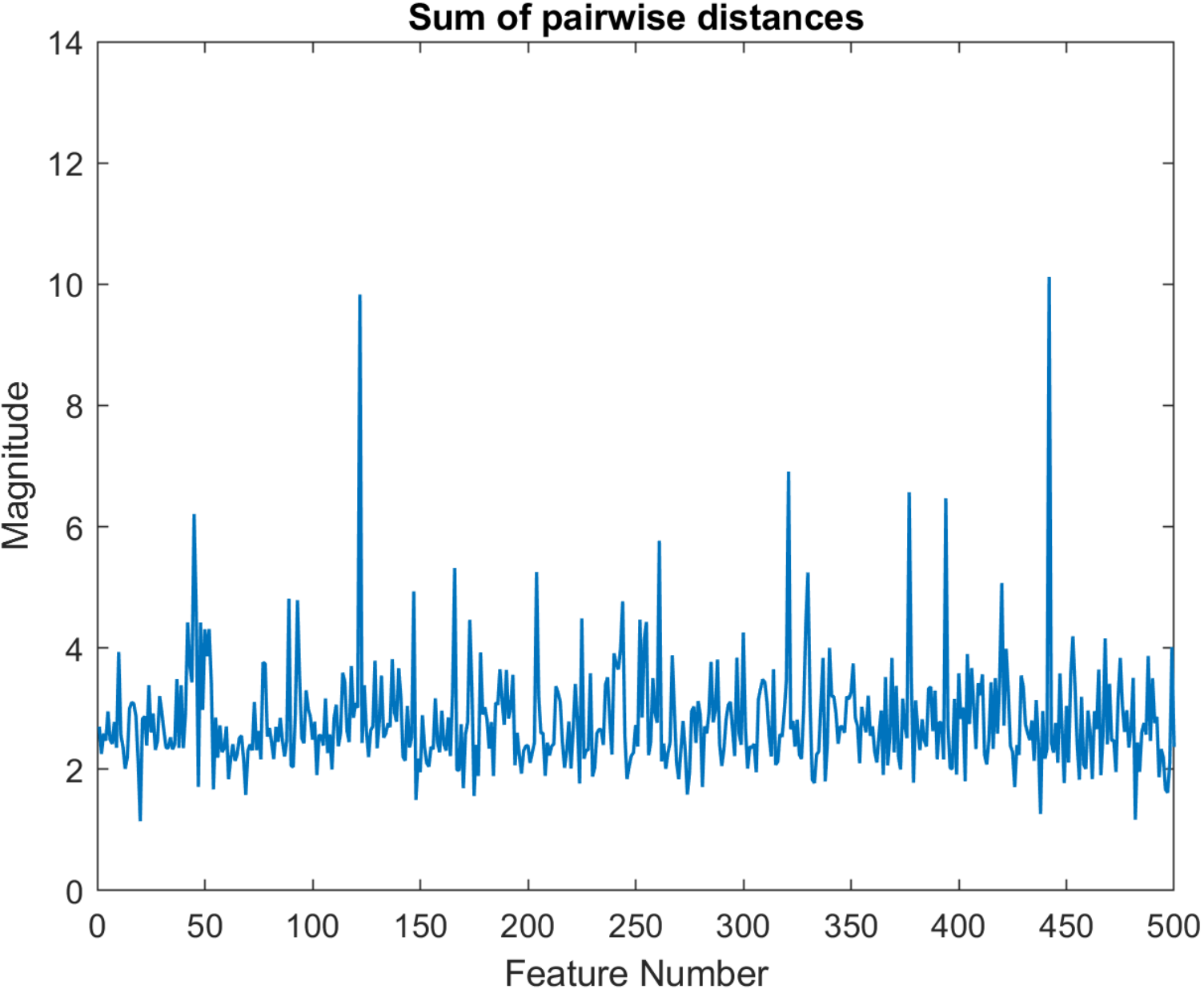}}\hfill
	\subfloat[]{\includegraphics[width=0.33\textwidth, height=0.2\textheight]{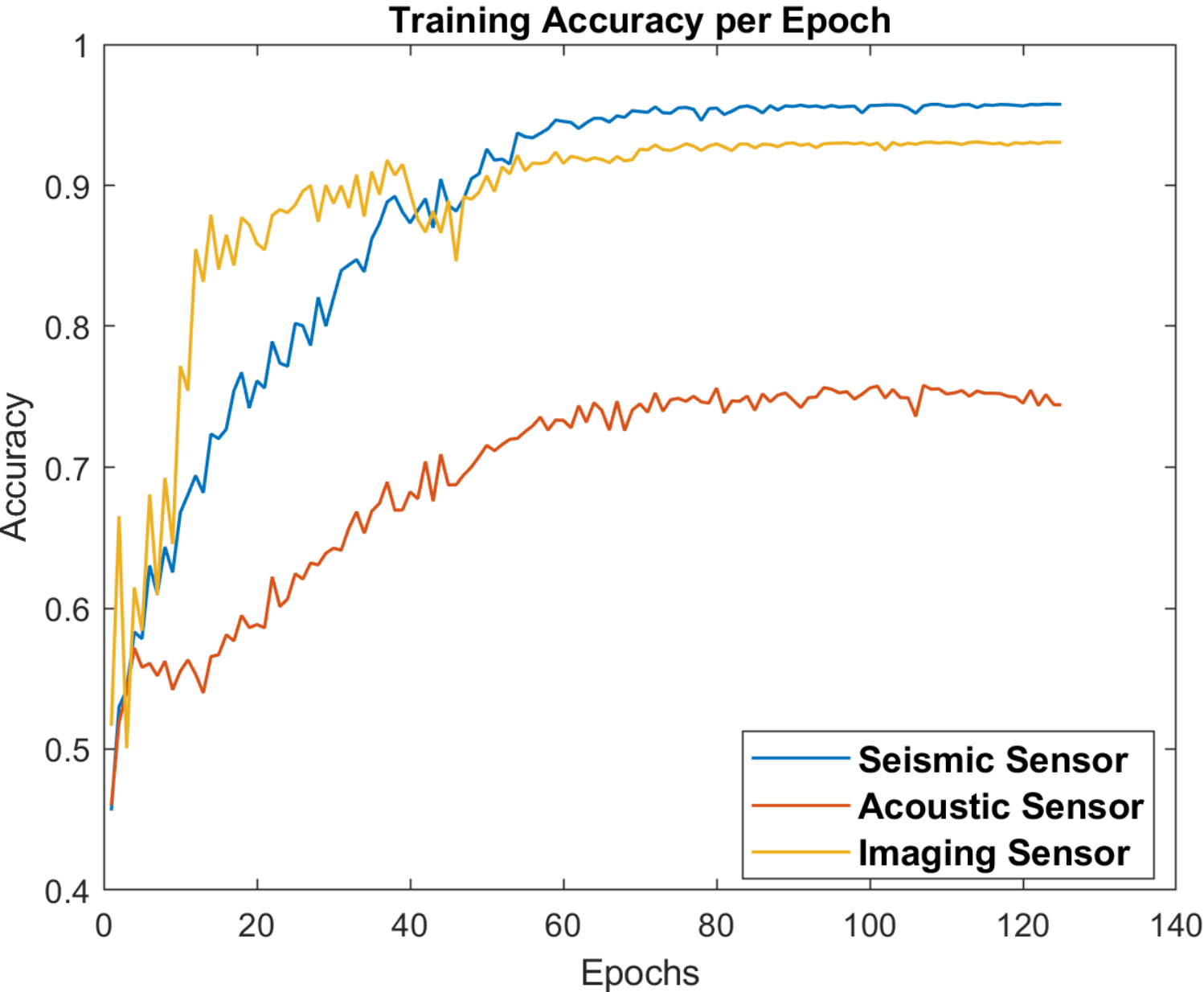}}\hfill
	\caption{(a): Optimization Losses, (b): Sum of pairwise distances between hidden estimates, (c): Training Accuracies/Epoch when all terms in Equation \ref{Final_Obj} are active.}
	\label{GAN+Comm+Sel}
\end{figure*}
\begin{figure*}[tbp]
	\centering
	\subfloat[Gaussian noise is added to Seismic and Acoustic Sensor, while Imaging data is assumed to be clean.]{\includegraphics[width=0.32\textwidth]{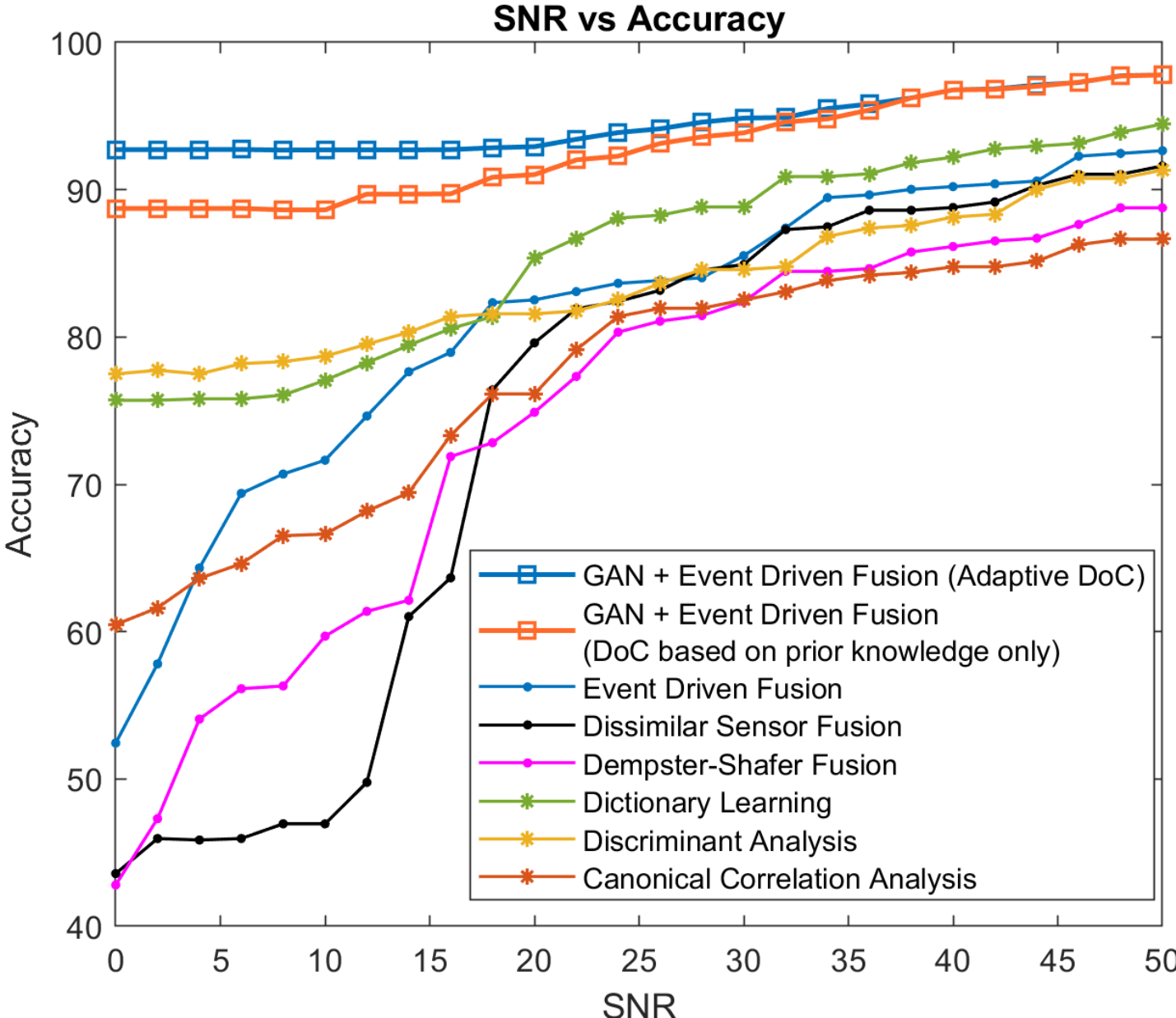}} \hfill
	\subfloat[Gaussian noise is added to Imaging and Acoustic Sensor, while Seismic data is assumed to be clean.]{\includegraphics[width=0.32\textwidth]{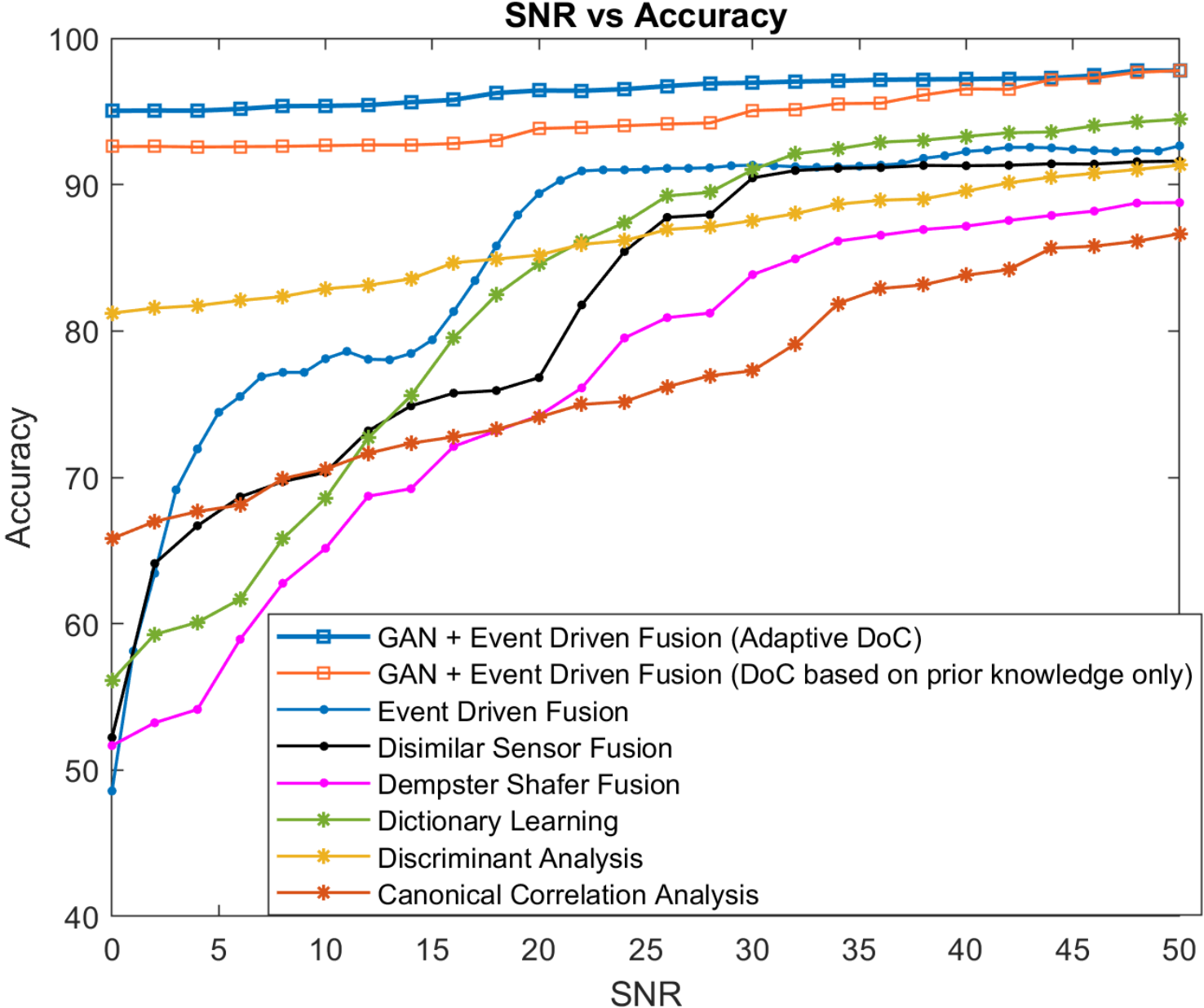}} \hfill
	\subfloat[Gaussian noise is added to Seismic and Imaging Sensor, while Acoustic data is assumed to be clean.]{\includegraphics[width=0.32\textwidth]{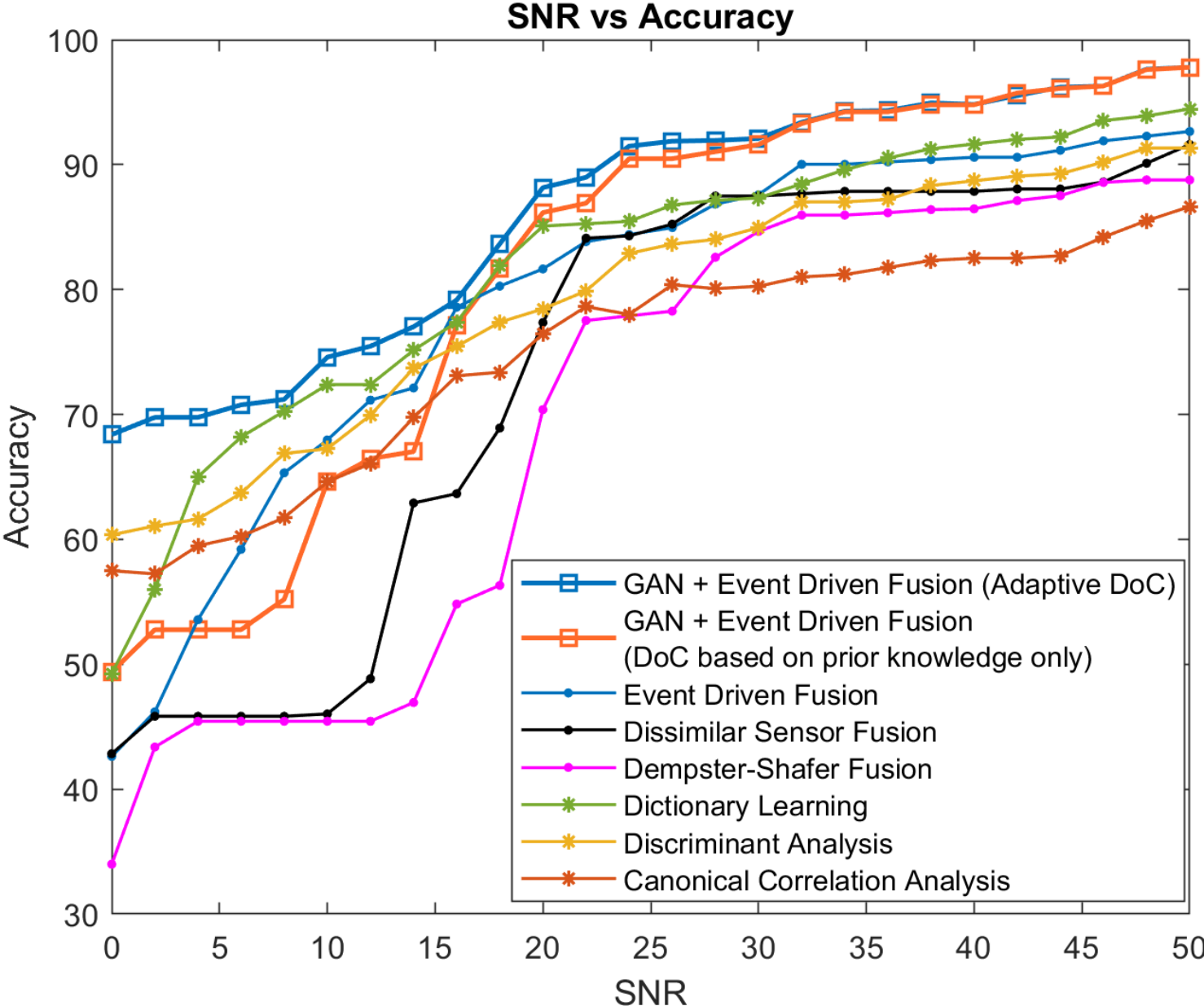}}
	\caption{Comparison of the proposed technique (GAN + Event Driven Fusion) with existing techniques for Dataset-1.}
	\label{SNRvsAcc_Army}
\end{figure*}
\begin{figure*}[tbp]
	\centering
	\subfloat[Gaussian noise is added to the Telescopic Imaging Sensor, while Radar data is assumed to be clean.]{\includegraphics[width=0.4\textwidth]{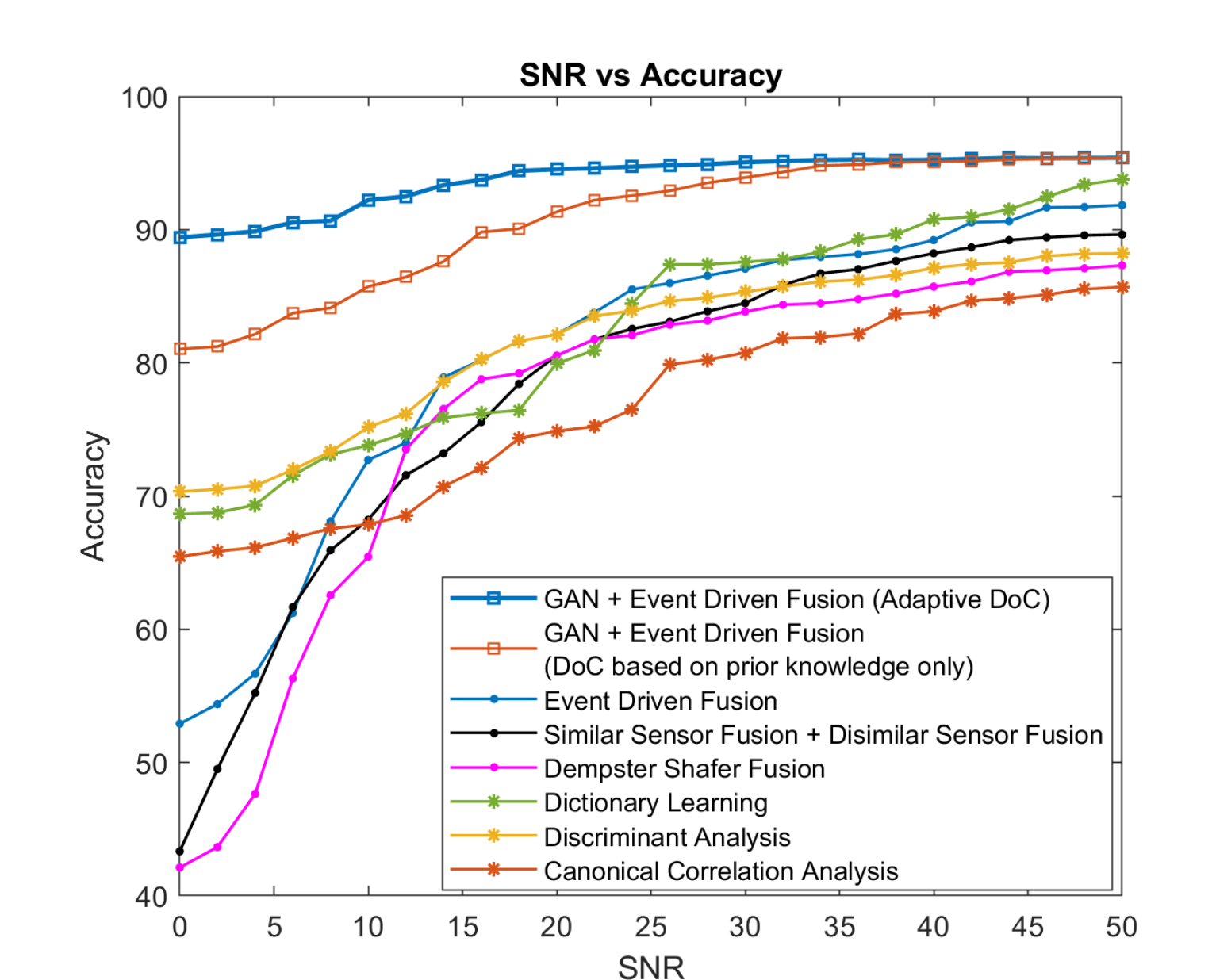}} \hfill
	\subfloat[Gaussian noise is added to the Radar Sensors, while Telescopic Imaging data is assumed to be clean.]{\includegraphics[width=0.4\textwidth]{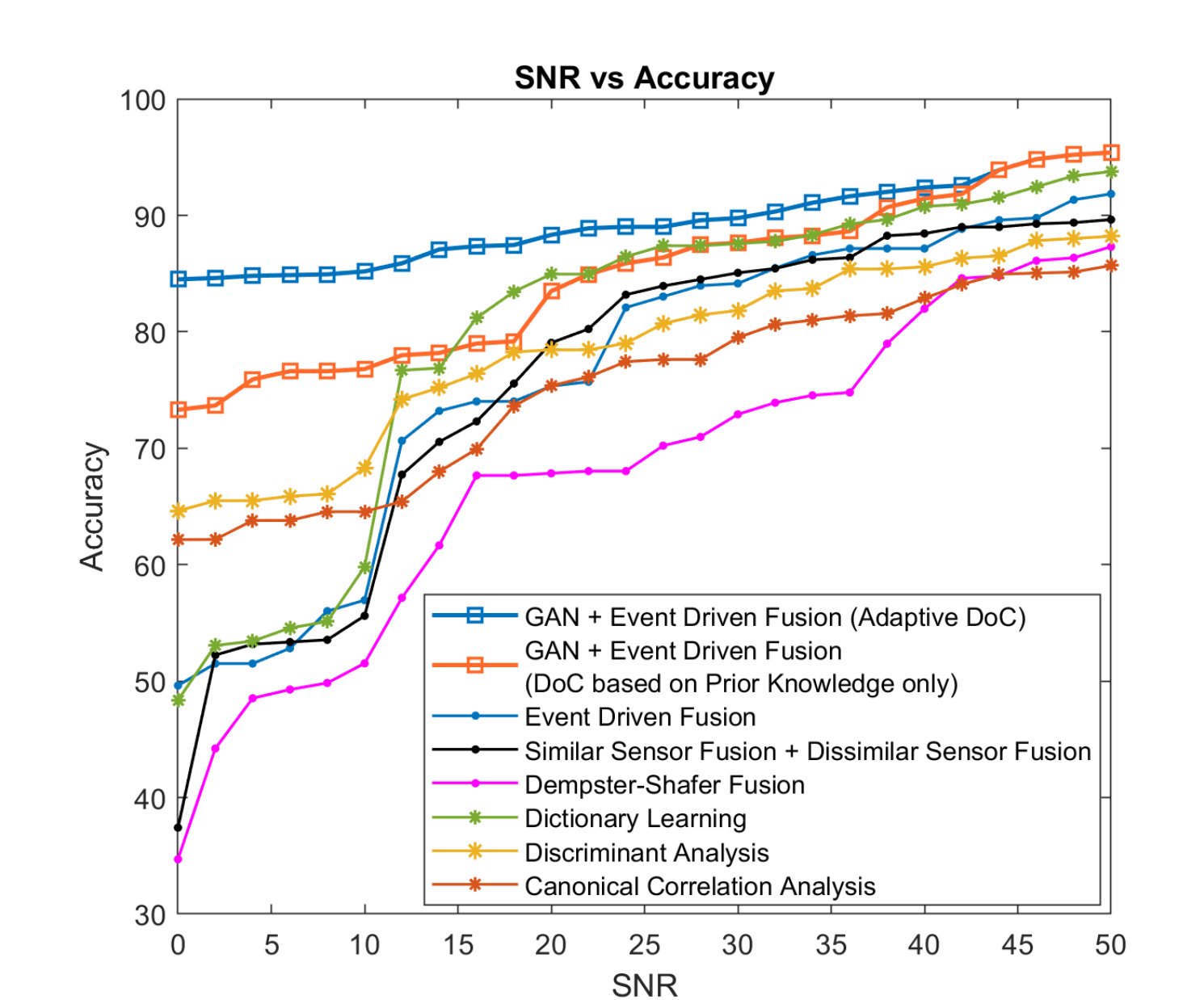}}
	\caption{Comparison of the proposed technique (GAN + Event Driven Fusion) with existing techniques for Dataset-2.}
	\label{PerfvsSNR_MDA}
\end{figure*}

Table \ref{acctable_ARL} shows the performance of different techniques as compared with the proposed approach of using a Generative Adversarial Network to learn optimal features, which are then used for target classification. All sensors are assumed to be working normally in Table \ref{acctable_ARL}. Comparisons are carried out with approaches that perform fusion at the decision level, as well as those that seek a hidden space for fusion. We find that using the individual sensor features (whose learning is driven by the existence of a structured hidden space) for classification and performing decision level fusion on these classifications yields better performance relative to those based solely on the hidden space. For the evaluation of Model Based Fusion approaches (see Section \ref{Model_Based_Fusion}), the dissimilar sensor setting is considered for Dataset-1 as all three sensors are different. On the other hand, for Dataset-2, the two radars are first combined using similar sensor fusion and the result of this is combined with the telescopic sensor using dissimilar sensor fusion.

We also compare the effects of different losses in our objective. 
First we implement a system which uses an Adversarial setup, along with the classification losses, i.e., $\gamma_1$, and $\gamma_2$ in Equation \ref{Final_Obj} are set to $0$. The losses are seen in Figure \ref{GAN_only}-(a). The blue plot represents the negative of the discriminator loss for the GAN network (y-axis on left), while the orange plot corresponds to the commutation loss (y-axis on right). Figure \ref{GAN_only}-(b) shows the sum of pairwise distances between the hidden estimates generated by the three sensors, i.e. $\text{Magnitude}(k) = \sum_{l,m=1}^{L}(\bm{H^l}(k)-\bm{H^m}(k))^2$, where, $k \in \{1,...,d_H\}$, refers to the feature number in the $d_H$-dimensional hidden space. Figure \ref{GAN_only}-(c) shows the individual classification performance of the sensors. It is observed that the performance of Seismic and Acoustic Sensors improves as the model is updated, but the performance of the imaging sensor is very low. 

Figure \ref{GAN+Comm} shows the optimization losses when the commutation term is added to the objective, i.e. only $\gamma_1$ is set to zero in Equation \ref{Final_Obj}. This amounts to minimizing of the pairwise commutation loss, hence forcing the hidden space estimates to lie in a common subspace. It is seen that including the commutation cost leads to closer hidden spaces (Figure \ref{GAN+Comm}-(b)).The performance of the imaging sensor, however, does not significantly increase. 

Figure \ref{GAN+Comm+Sel} shows the optimization loss when all the terms in Equation \ref{Final_Obj} are active. The imaging sensor now starts giving better performance. Since the $L_{\infty,1}$ norm on the selection matrix allows representation of private/shared features in the hidden space, the heterogeneity of the imaging sensor is maintained, and selected features are more optimal for classification based on the imaging sensor. The effects of the $L_{\infty,1}$ norm can also be seen in the differences between the hidden estimates in Figure \ref{GAN+Comm+Sel}-(b), where some features now exhibit more distinguishing diversity compared to others, as they may correspond to private features. 

\subsection{Robustness Analysis}
\label{RA_ARL}
The major advantage of learning a hidden space between the modalities is the ability to detect sensor damage in operation, and to generate representative features for that damaged sensor as seen in Equation \ref{Reconstruction}. The contribution of the representative features towards the fused decision can also be controlled via the degree of confidence. 

We also study how the performance of the system varies with different Signal to Noise Ratios. That is noise is selectively added to some sensors while another is normally functioning. 

Figures \ref{SNRvsAcc_Army} and \ref{PerfvsSNR_MDA} show the degradation of the fusion performance as the SNR decreases for Dataset-1 and Dataset-2, respectively. The plots with `asterisk'(*) markers represent approaches that search for a common subspace in order to fuse the different modalities, while those with `dot'(.) markers represent approaches that perform fusion at the decision level. Finally, the plots with `square' markers show the performance of our proposed approach. The blue plot uses adaptive DoC $\left(DoC^l_t = (1-p_D(\hat{h}_t^l)).Acc_{train}^l\right)$, while the orange plot only uses the prior information $\left(DoC^l_t = Acc_{train}^l\right)$, which is also the case for other plots using decision level fusion. It is observed that it is important to update the DoC during the implementation based on the sensor condition, rather than only depending on the prior information about the discriminative abilities of the sensor. 

Furthermore, the above graphs show that the degradation has severe effects in the case where the seismic and imaging sensors are damaged. This is due to the fact that the discriminative power of the acoustic sensor is low, and in spite of adapting the DoC and generating representative features, the performance is limited by the information contained in the observations of the sensor. A similar trend is also observed for Dataset-2, where the telescopic imagery has lower discriminative ability compared to the radar sensors.

\section{Conclusion}
In this work, we proposed a data driven approach to learn a structured hidden space between sensors by way of a bank of Generative Adversarial Networks. The maps into the hidden space are forced to commute with each other, leading to the estimates from different modalities to lie in a common subspace. Enforcing commutation is also shown to lead to faster convergence and better hidden space estimates.
The hidden space is subsequently used to learn the features of a target of interest for classification. The `private' and `shared' information in a hidden space was further enhanced by including a selection matrix that selects features of interest for each modality before classification. The hidden space serves a dual purpose of detecting noisy/damaged sensors and mitigating sensor losses by ensuring a graceful performance degradation. Experiments on multiple datasets show that the proposed approach secures a great degree of robustness to noisy/damaged sensors, and outperforms existing fusion algorithms. 


%

\ifCLASSOPTIONcaptionsoff
  \newpage
\fi

\newpage
\section*{Appendix}
\section*{Algorithms}

\begin{algorithm}
	\label{CGAN_Update}
	\textbf{Algorithm 1: } Training the CGAN system \newline
	Let $\{\theta_{g^l}^q\}_{q \in \{1,...,Q\}}$ be the parameters for the $q^{th}$ layer of the $l^{th}$ generator, and $\bm{y_q}$ be the output of that layer. \newline $\bm{y^Q} = \bm{Z^l}[\bm{M^l}(\bm{X^l})]$, and,  $\bm{y^{Q-1}} = \bm{M^l}(\bm{X^l})$.\newline
	Similarly, Let $\{\theta_d^r\}_{r \in \{ 1,...,R \}}$ be the parameters of the $r^{th}$ layer of the discriminator and $\bm{z^r}$ be the output of that layer.
	\begin{itemize}
		\item for $j$ in $1:\text{Number of Iterations}$
		\begin{itemize}
			\item for $l$ in $1:L$
			\begin{itemize}
				\item\small{Update the discriminator network,}
				\begingroup\makeatletter\def\f@size{6}\check@mathfonts
				\begin{equation}
				\begin{split}
					\theta_d^{r^{(j)}} \leftarrow \theta_d^{r^{(j-1)}} + \mu_D \left\{ \frac{d\mathcal{L}(\bm{G^l}, \bm{D}, \bm{Z^l}, \bm{S^l}, \bm{C^l})}{d\bm{z^R}}\frac{d\bm{z^R}}{d\bm{z^{R-1}}}  ...\frac{d\bm{z^{r+1}}}{d\theta_d^r} \right\}
				\end{split}
				\end{equation}
				\endgroup
				
				\item\small{Update the $l^{th}$ generator network,}
				\begingroup\makeatletter\def\f@size{7}\check@mathfonts
				\begin{equation}
				\bm{C^{l^{(j)}}} \leftarrow  \bm{C^{l^{(j-1)}}} - \mu_G \left\{  \frac{d \mathcal{L}(\bm{G^l}, \bm{D}, \bm{Z^l}, \bm{S^l}, \bm{C^l})}{d\bm{C^l}} \right\}
				\end{equation}
				\endgroup
				\begingroup\makeatletter\def\f@size{6.5}\check@mathfonts
				\begin{equation}
				\begin{split}
					\bm{S^{l^{(j)}}} \leftarrow \bm{S^{l^{(j-1)}}} - \mu_G & \left\{ \frac{d \mathcal{L}(\bm{G^l}, \bm{D}, \bm{Z^l}, \bm{S^l}, \bm{C^l})}{d\bm{C^l}(\bm{F^l})} \frac{d\bm{C^l}(\bm{F^l})}{d\bm{F^l}} \frac{d\bm{F^l}}{\bm{S^l}} \right. \\ 
					& \left. 
					+ \frac{d \mathcal{L}(\bm{G^l}, \bm{D}, \bm{Z^l}, \bm{S^l}, \bm{C^l})}{d\bm{S^l}} \right\}
				\end{split}
				\end{equation}
				\endgroup
				\begingroup\makeatletter\def\f@size{6}\check@mathfonts
				\begin{equation}
				\begin{split}
				\bm{Z^{l^{(j)}}} \leftarrow \bm{Z^{l^{(j-1)}}} - \mu_G & \left\{ \frac{d \mathcal{L}(\bm{G^l}, \bm{D}, \bm{Z^l}, \bm{S^l}, \bm{C^l})}{d\bm{C^l}(\bm{F^l})}  \right. \\ 
				&\frac{d\bm{S^l}(\bm{Z^l}[\bm{M^l}(\bm{X^l})])}{d\bm{Z^l}[\bm{M^l}(\bm{X^l})]} \frac{d\bm{C^l}(\bm{F^l})}{d\bm{F^l}} \frac{d\bm{Z^l}[\bm{M^l}(\bm{X^l})]}{d\bm{Z^l}} \\
				& +
				\frac{d\mathcal{L}(\bm{G^l}, \bm{D}, \bm{Z^l}, \bm{S^l}, \bm{C^l})}{d\bm{G^l}(\bm{X^l})} \frac{d\bm{G^l}(\bm{X^l})}{d\bm{Z^l}} \\ 
				& \left. +
				\frac{d\mathcal{L}(\bm{G^l}, \bm{D}, \bm{Z^l}, \bm{S^l}, \bm{C^l})}{d\bm{Z^l}} \right\}
				\end{split}
				\end{equation}
				\endgroup
				\begingroup\makeatletter\def\f@size{6}\check@mathfonts
				\begin{equation}
				\begin{split}
				\theta_{g^l}^{q^{(j)}} \leftarrow \theta_{g^l}^{q^{(j-1)}} - \mu_G &\left\{ \frac{d\mathcal{L}(\bm{G^l}, \bm{D}, \bm{Z^l}, \bm{S^l}, \bm{C^l})}{d\bm{C^l}(\bm{F^l})} \frac{d\bm{C^l}(\bm{F^l})}{d\bm{F^l}} \right. \\  & \frac{d\bm{S^l}(\bm{G^l}(\bm{X^l}))}{d\bm{G^l}(\bm{X^l})} \frac{d\bm{G^l}(\bm{X^l})}{d\bm{y^{Q-1}}} \\ 
				& \left. + \frac{d\mathcal{L}(\bm{G^l}, \bm{D}, \bm{Z^l}, \bm{S^l}, \bm{C^l})}{d\bm{G^l}(\bm{X^l})} \frac{d\bm{G^l}(\bm{X^l})}{d\bm{y^{Q-1}}} \right\} 	\\
				&\left\{\frac{d\bm{y^{Q-1}}}{\bm{dy^{Q-2}}}...\frac{d\bm{y^{q+1}}}{d\theta_{g^l}^q}\right\}
				\end{split}
				\end{equation}
				\endgroup
				
			\end{itemize}
		\end{itemize}
	\end{itemize}
\end{algorithm}

\begin{algorithm}
	\textbf{Algorithm 2: } \small{Hierarchical Clustering (Agglomerative Clustering)}
	\begin{itemize}
		\item Initialize clusters at $v = 0$: $\mathcal{C}_0 = \{ Z_0^j = \{h_j\}, j = 1...L.N \}$
		\item while $\text{Number of Clusters} > 1$:
		\item \quad v = v + 1
		\item \quad Among all cluster pairs, $\{Z_{v-1}^r, Z_{v-1^s}\}$, find the one, say $\{Z_{v-1}^i, Z_{v-1}^j\}$, such that:
		\begin{equation}
		d(Z_{v-1}^i, Z_{v-1}^j) = \min_{r,s} d(Z_{v-1}^r, Z_{v-1}^s),
		\end{equation}  
		where, $d(.)$ is a measure of dissimilarity. 
		\item \quad Assign $Z_v^q = Z_{v-1}^i \cup Z_{v-1}^j$
		\begingroup\makeatletter\def\f@size{9}\check@mathfonts
		\item \quad Get new clustering: $\mathcal{C}_v = \{ (\mathcal{C}_{v-1} - \{ Z_{v-1}^i,Z_{v-1}^j \}) \cup Z_v^q \}$
		\endgroup 
	\end{itemize}
\end{algorithm}




%

%
\begin{IEEEbiography}[{\includegraphics[width=1in,height=1.15in,clip,keepaspectratio]{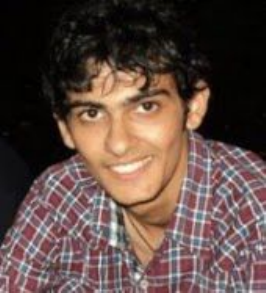}}]{Siddharth Roheda}
	\scriptsize received the B.Tech degree in Electronics and Communication from Nirma University, India, in 2015 and completed his M.Sc. and Ph.D. degree with the Department of Electrical and Computer Engineering at North Carolina State University, Raleigh, NC, USA in 2020. His current research interests include Information Fusion, Machine Learning, Deep Learning, Computer Vision, and Signal Processing. 
\end{IEEEbiography}

\begin{IEEEbiography}[{\includegraphics[width=1in,height=1.15in,clip,keepaspectratio]{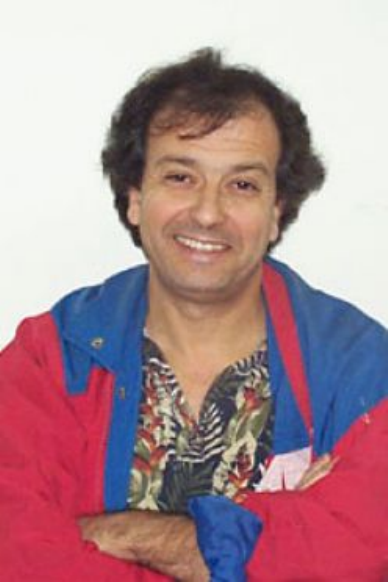}}]{Hamid Krim}
	\scriptsize received the B.Sc. and M.Sc. and Ph.D. in ECE. He was
	a Member of Technical Staff at AT\&T Bell Labs,
	where he has conducted research and development
	in the areas of telephony and digital communication systems/subsystems. Following an NSF Postdoctoral Fellowship at Foreign Centers of Excellence,
	LSS/University of Orsay, Paris, France, he joined the
	Laboratory for Information and Decision Systems, MIT, Cambridge, MA, USA, as a Research Scientist and where he performed/supervised research.
	He is currently a Professor of electrical engineering in the Department of
	Electrical and Computer Engineering, North Carolina State University,
	NC, leading the Vision, Information, and Statistical Signal Theories and
	Applications Group. His research interests include statistical signal and image
	analysis and mathematical modeling with a keen emphasis on applied problems
	in classification and recognition using geometric and topological tools. He has
	served on the SP society Editorial Board and on TCs, and is the SP Distinguished
	Lecturer for 2015-2016.
\end{IEEEbiography}
\begin{IEEEbiography}[{\includegraphics[width=1in,height=1.15in,clip,keepaspectratio]{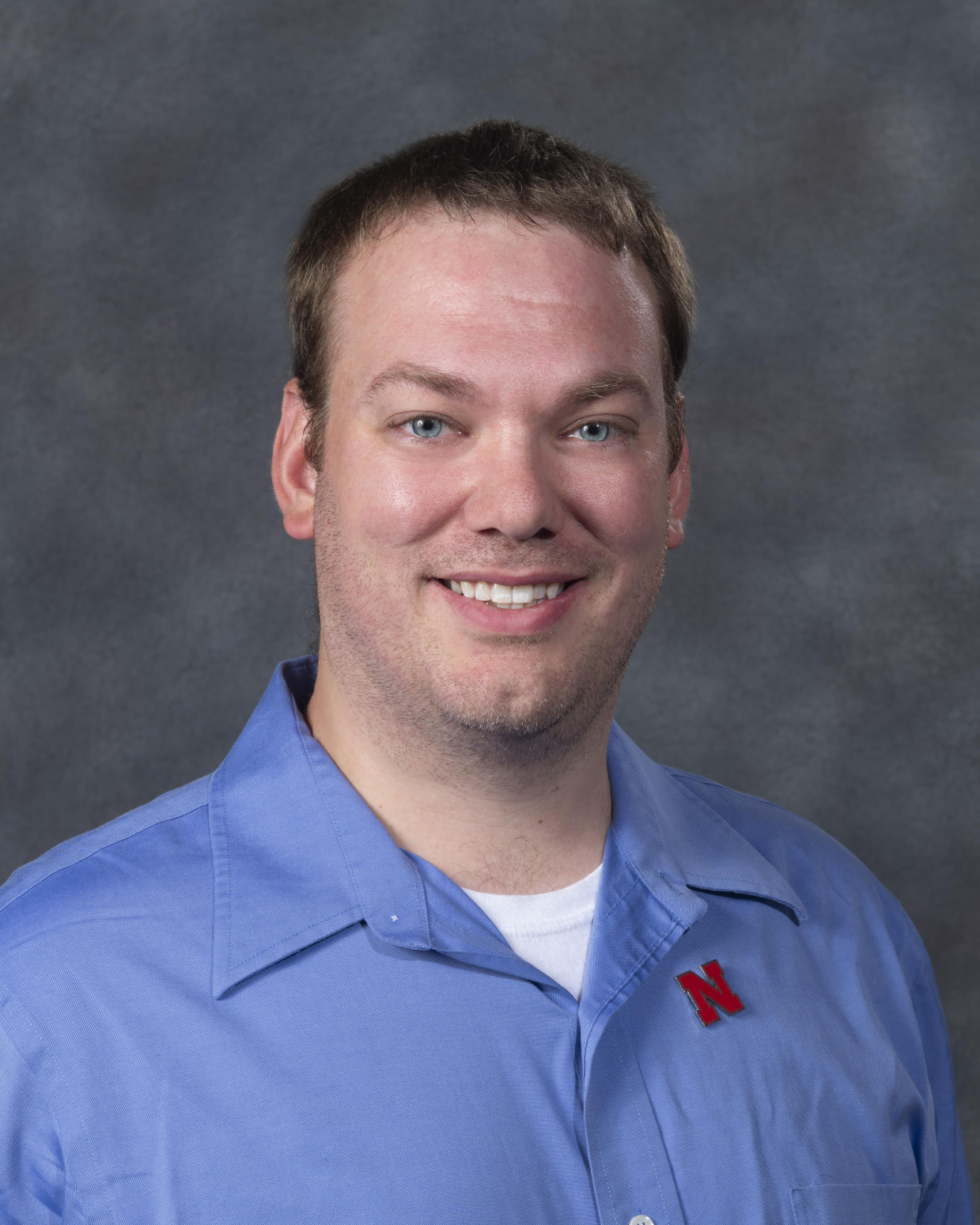}}]{Benjamin S. Riggan}
	\scriptsize received the B.S. degree in computer engineering and the M.S. and Ph.D. degrees in Electrical Engineering from North Carolina State University in 2009, 2011, and 2014, respectively. Currently, he is an assistant professor in the Department of Electrical and Computer Engineering at the University of Nebraska-Lincoln.  Prior to joining the Electrical and Computer Engineering Department at UNL, he worked for the U.S. Army Research Laboratory’s Image Processing and Networked Sensing and Fusion Branches, focusing on cross-domain facial recognition and multi-modal analytics. His research interests are in the areas of computer vision, image and signal processing, and biometrics/forensics that are related to domain adaptation, multi-modal analytics, and machine learning. He received a best paper award at the IEEE Winter Conference on Applications of Computer Vision (WACV) in 2016, a runner-up best paper award at the IEEE International Conference on Biometrics: Theory, Applications, and Systems (BTAS) in 2016, and a best paper award at IEEE WACV in 2018. He has also organized workshops and tutorials at the IEEE International Conference on Automatic Face and Gesture Recognition (FG) in 2017, IEEE/IAPR Joint Conference on Biometrics (IJCB) in 2017,  IEEE WACV in 2018 and 2019, and BTAS 2019 on topics related to heterogeneous facial recognition and cross-domain biometric recognition.
\end{IEEEbiography}

\end{document}